\documentclass[11pt]{article}

\usepackage[a4paper,margin=1in]{geometry}
\usepackage{lmodern}
\usepackage{microtype}
\usepackage{amsmath,amsthm,amssymb,amsfonts,mathtools}
\usepackage{bm}
\usepackage{bbm}
\usepackage{authblk}
\usepackage{graphicx}
\usepackage{booktabs}
\usepackage[numbers,sort&compress]{natbib}
\usepackage{hyperref}
\usepackage[nameinlink]{cleveref}
\usepackage{enumitem}
\usepackage{float}
\usepackage[normalem]{ulem}
\setlist{nosep}

\hypersetup{colorlinks=true,linkcolor=blue,citecolor=magenta,urlcolor=blue}

\newtheorem{theorem}{Theorem}[section]
\newtheorem{proposition}[theorem]{Proposition}

\newtheorem{remark}[theorem]{Remark}
\newtheorem{corollary}[theorem]{Corollary}
\graphicspath{{figures/}{img/}} 

\usepackage[dvipsnames]{xcolor} 

\usepackage{braket}

\DeclarePairedDelimiterX{\inp}[2]{\langle}{\rangle}{#1, #2}

\title{Kriging via Variably Scaled Kernels}

\author[1]{\normalsize Audone G}
\author[2]{\normalsize Marchetti F}
\author[3]{\normalsize Perracchione E}
\author[4]{\normalsize Rossini M}

\affil[1]{\small Faculty of Engineering and Physical Sciences, University of Southampton, United Kindom, \hfill\texttt{g.audone@soton.ac.uk}}
\affil[2]{\small Dipartimento di Matematica \lq\lq Tullio Levi-Civita", Università di Padova, Italy \texttt{francesco.marchetti@unipd.it}}
\affil[3]{\small Dipartimento di Scienze Matematiche Giuseppe Luigi Lagrange, Politecnico di Torino, Italy, \texttt{emma.perracchione@polito.it}}
\affil[4]{\small Università degli Studi di Milano-Bicocca, Dipartimento di Matematica e Applicazioni, Italy, \texttt{milvia.rossini@unimib.it}}

\date{\today}

\begin{document}
\maketitle

\begin{abstract}
Classical Gaussian processes and Kriging models are commonly based on stationary kernels, whereby correlations between observations depend exclusively on the relative distance between scattered data. While this assumption ensures analytical tractability, it limits the ability of Gaussian processes to represent heterogeneous correlation structures. In this work, we investigate variably scaled kernels as {an effective tool}  for constructing non-stationary Gaussian processes by explicitly modifying the correlation structure of the data. Through a scaling function, variably scaled kernels alter the correlations between data and enable the modeling of targets exhibiting abrupt changes or discontinuities. We analyse the resulting predictive uncertainty via the variably scaled kernel power function and clarify the relationship between variably scaled kernels-based constructions and classical non-stationary kernels. Numerical experiments demonstrate that variably scaled kernels-based Gaussian processes yield improved reconstruction accuracy and provide uncertainty estimates that reflect the underlying structure of the data. 
\end{abstract}

\section{Introduction}

Gaussian Processes (GPs) and Kriging models constitute a well-established framework in statistical learning for regression and classification problems \cite{Kozak2019,rasmussen2005gaussian,stein1999interpolation,cressie1993statistics}. In their standard formulation, GPs are defined by stationary covariance kernels, meaning that the covariance between any two observations depends only on their  corresponding locations. This assumption yields a number of desirable analytical properties, including a direct connection with Reproducing kernel Hilbert Space (RKHS) theory \cite{Berlinet2004rkhs,Aronszajn1950rkhs,scheuerer2013interpolation}.
Beyond their classical statistical interpretation, GPs can also be viewed as a principled framework for approximating an unknown function from a given sample. From this perspective, the choice of the covariance kernel encodes prior assumptions on smoothness and correlation, while conditioning on the observed data produces a predictor that acts as an interpolant in the noise-free case and as a regularized estimator in the presence of noise.
However, stationarity, as assumed in simple Kriging, imposes a spatially homogeneous correlation structure and therefore limits the class of functions that can be accurately represented. In particular, stationary kernels are not well suited to describe targets exhibiting localized variations in regularity or abrupt transitions \cite{goldberg1998heteroscedastic,kersting2007mlhgp,lazaro2011vhgp}. To overcome these limitations, several non-stationary constructions have been proposed, including kernels with spatially varying parameters \cite{paciorek2004nonstationary,paciorek2006environmetrics}, flexible kernel constructions such as spectral mixtures \cite{wilson2013sm}, and hierarchical models based on domain partitioning, such as treed Gaussian processes \cite{gramacylee2008tgp}.

To relax the stationarity assumption in GPs, in this paper we propose a valid and effective alternative based on Variably Scaled Kernels (VSKs) \cite{bozzini2015interpolation,rossini2022overview,campi2021vsk}, which offer a natural and mathematically structured mechanism for introducing adaptivity into kernel based models.
A VSK is constructed by augmenting the input space through a scaling map that adds extra coordinates to each point, and by evaluating a stationary kernel in the resulting lifted space. Although the kernel remains stationary in this augmented space, the induced kernel on the original domain becomes non-stationary. The scaling map effectively modifies the geometry of the input domain, altering pairwise distances and therefore changing the correlation structure between data points.
When embedded into the Gaussian process framework, this construction yields a non-stationary covariance function in a  simple and analytically transparent way. Variably Scaled Kernels thus provide a flexible tool for Kriging regression in heterogeneous settings, allowing the model to adapt locally to variations in smoothness or spatial dependence while preserving the probabilistic structure of Gaussian process modeling.
The proposed approach leads to a class of non stationary processes characterized by geometry induced correlation structures.
 We analyse the associated covariances, establish explicit connections with classical non-stationary kernel constructions, and investigate predictive uncertainty through the VSK power function. Numerical experiments on synthetic datasets illustrate the behaviour of the proposed models and their ability to represent localized irregularities within a Gaussian process setting.

The paper is organized as follows. In Section \ref{preliminari}, we review the main concepts of kernel interpolation and their relationship with Kriging and GPs. Section \ref{VSK-GP} constitutes the core of the paper and develops the connection between variably scaled kernels in approximation theory and Gaussian process modeling. Numerical experiments are presented in Section \ref{sec:examples}, and conclusions are drawn in Section \ref{conclusions}.

\section{Statistics and approximation with kernels}
\label{preliminari}
Let $X=\{x_i\}_{i=1}^N\subset\Omega$, with $\Omega\subset\mathbb{R}^d$, be a set of pairwise distinct data sites, and let
$\{f_i\}_{i=1}^N := \{f(x_i)\}_{i=1}^N$ be the corresponding observed responses, which we assume to be sampled from an (unknown) function
$f:\Omega\to\mathbb{R}$. The goal of an approximation procedure is to construct a function $s_X:\Omega\to\mathbb{R}$ that approximates the underlying data-generating process, whereas in the interpolation setting, the stronger requirement $s_X(x_i)=f_i$ is imposed for all $i=1,\dots,N$.

A classical approach constructs $s_X$ as a linear combination of coefficients and an associated basis
$\{b_i(\cdot)\}_{i=1}^N$. Uniqueness of such a representation is ensured if $\{b_i(\cdot)\}_{i=1}^N$ forms a Haar system.
For $d>1$, this condition holds only for trivial Haar spaces, i.e., spaces spanned by a single function
\cite[Theorem 2.3, p.~19]{wendland2005scattered}. Nevertheless, when one considers data-dependent bases, as in kernel-based methods,
existence and uniqueness of the approximant can be guaranteed under suitable assumptions both on the kernel and  the data.

\subsection{Kernels in deterministic settings}\label{sec:appr}

Let $\Omega\subseteq\mathbb{R}^d$ and let $\kappa:\Omega\times\Omega\to\mathbb{R}$ be a symmetric positive definite kernel.
Associated with $\kappa$ there exists a unique reproducing kernel Hilbert space (RKHS) $\mathcal{H}$,
endowed with inner product $(\cdot,\cdot)_{\mathcal{H}}$, such that:
\begin{enumerate}[label=(\roman*)]
\item $\kappa(\cdot,x)\in\mathcal{H}$ for all $x\in\Omega$;
\item $\kappa$ satisfies the reproducing property:
\[
(g,\kappa(\cdot,x))_{\mathcal{H}}=g(x),
\qquad \forall\, g\in\mathcal{H},\ \forall\, x\in\Omega.
\]
\end{enumerate}
Consequently, the kernel can be expressed through the inner product as
\begin{equation*}
(\kappa(\cdot,x),\kappa(\cdot,x'))_{\mathcal{H}}=\kappa(x,x'),
\qquad x,x'\in\Omega.
\end{equation*}

More precisely, we consider rescaled kernels by a positive factor $\ell>0$, defined as
\begin{equation*}
\kappa_\ell(x,x')=\kappa\!\left(\frac{x}{\ell},\,\frac{x'}{\ell}\right),
\qquad x,x'\in\Omega,
\end{equation*}
where $\ell$ is usually referred to as the \emph{shape} or \emph{length-scale} parameter. 

In the interpolation setting, i.e., when exact data matching is required, the approximant is sought in the form
\begin{equation}\label{interpolante}
s_X(x)=\sum_{i=1}^N \alpha_i\,\kappa_\ell(x,x_i),
\qquad x\in\Omega,
\end{equation}
where, letting $y=(f(x_1),\dots,f(x_N))^\intercal$ and $\alpha=(\alpha_1,\dots,\alpha_N)^\intercal$,
the coefficients are determined by imposing the interpolation conditions, namely by solving the linear system
\begin{equation}\label{glob_sys}
K_{\ell}\,\alpha=y,
\qquad (K_{\ell})_{ij}=\kappa_\ell(x_i,x_j),
\qquad i,j=1,\ldots,N.
\end{equation}
Equivalently, letting ${\rm k}_{\ell}(x)=(\kappa_\ell(x,x_1),\dots,\kappa_\ell(x,x_N))^\intercal$, we may write
\begin{equation}\label{val_atteso}
s_X(x)={\rm k}_{\ell}(x)^\intercal K_{\ell}^{-1}y.
\end{equation}

Both \eqref{interpolante} and \eqref{glob_sys} require the invertibility of the kernel matrix $K_{\ell}$.
For positive definite kernels, this holds as long as $X$ is $\kappa_\ell$-unisolvent, i.e., if the family of functions
$\{\kappa_\ell(\cdot,x_1),\dots,\kappa_\ell(\cdot,x_N)\}$ is linearly independent. In particular,
\[
\sum_{i=1}^N c_i\,\kappa_\ell(\cdot,x_i)=0
\quad\Longrightarrow\quad
c_1=\cdots=c_N=0,
\]
which is equivalent to the nonsingularity of $K_{\ell}$. A sufficient condition is that $\kappa_\ell$ be
\emph{strictly} positive definite, in which case any set of pairwise distinct points is unisolvent.

Deterministic pointwise error bounds for kernel interpolation involve the so-called \emph{power function}
\begin{equation*}
P_{\kappa_\ell,X}(x)
=
\sqrt{\kappa_\ell(x,x)-{\rm k}_{\ell}(x)^\intercal K_{\ell}^{-1}{\rm k}_{\ell}(x)}.
\end{equation*}
In particular, classical bounds are of the form
\begin{equation*}
|f(x)-s_X(x)|
\;\le\;
P_{\kappa_\ell,X}(x)\,\|f\|_\mathcal{H},
\qquad x\in\Omega,\ \ f\in\mathcal{H}.
\end{equation*}

We further remark that the interpolant defined in \eqref{interpolante} is the minimum-norm interpolant in the associated RKHS, i.e.,
\begin{equation}\label{eq:minnorm}
s_X
=
\arg\min_{g\in\mathcal{H}} \|g\|_\mathcal{H}^2
\quad \text{subject to } g(x_i)=f_i,\ \ i=1,\dots,N.
\end{equation}

On the other hand, when the data are noisy or smoothing is desired, one seeks a smoothing spline (or a Tikhonov-regularized estimator)
associated with a quadratic loss. Namely, instead of solving \eqref{eq:minnorm}, one considers
\begin{equation*}
s_X^\lambda
=
\arg\min_{g\in\mathcal{H}}
\sum_{i=1}^N \big(g(x_i)-f_i\big)^2
+
\lambda\,\|g\|_\mathcal{H}^2,
\end{equation*}
where $\lambda>0$ is a regularization parameter. By the representer theorem (see, e.g., \cite{smola-et-al:2001}),
the solution has the form \eqref{interpolante}, but the coefficients now satisfy the regularized system
\begin{equation*}
(K_{\ell}+\lambda I)\alpha=y.
\end{equation*}
This yields
\begin{equation*}
s_X^\lambda(x)={\rm k}_{\ell}(x)^\intercal (K_{\ell}+\lambda I)^{-1}y.
\end{equation*}
The regularization parameter $\lambda$ controls the trade-off between data fidelity (small residuals) and smoothness (small RKHS norm).
The limiting case $\lambda=0$ recovers the interpolation solution \eqref{val_atteso}, provided that $K_{\ell}$ is invertible.

\subsection{Kernels in statistical settings}
To introduce the so-called \emph{simple Kriging} approach (see, e.g., \cite{Fasshauer-et-al:2015,scheuerer2013interpolation}), we interpret $f$ as a realization of a random field
$F=\{F_x,\,x\in\Omega\}$, where the $F_x$ are random variables defined on a common probability space, with
$\mathbb{E}[F_x]=0$ and finite second moment for all $x\in\Omega$.
Under these assumptions, the covariance function $\phi$ is a symmetric positive definite kernel.
More precisely, the observed values $y$ are interpreted as a realization of the random vector
$Y=(F_{x_1},\dots,F_{x_N})^\intercal$.
Assuming that $F$ is a Gaussian process with zero mean and covariance function $\phi$, we write
\[
F \sim \mathcal{GP}(0,\phi(x,x'))\quad x,x' \in \Omega.
\]

In many situations, the observations are affected by additive noise. Let $\varepsilon(x)$ denote the additive noise random variable associated with each $x\in\Omega$, with zero mean and variance $\sigma_n^2$, independent of $F$, and such that the noise variables at different locations are mutually independent.
Then the covariance function takes the form (see, e.g., \cite{rasmussen2005gaussian,Schaback_Wendland_2006})
\begin{equation*}
\phi(x,x')+\sigma_n^2\,\delta_{xx'},
\end{equation*}
where $\delta_{xx'}$ denotes the Kronecker delta.

In our setting, following the notation of Section~\ref{sec:appr}, we consider the covariance function
\begin{equation}\label{eq:cov-rasmussen}
\phi_{\ell,f,n}(x,x')=\sigma_f^2\,\kappa_\ell(x,x')+\sigma_n^2\,\delta_{xx'},
\end{equation}
where $\kappa_\ell$ is a normalized kernel, and
$\sigma_f>0$ denotes the process standard deviation.

The formulation \eqref{eq:cov-rasmussen} incorporates both the latent process variability (through $\sigma_f^2$)
and the observational noise (through $\sigma_n^2$).
In particular, when $\kappa_\ell(x,x)=1$ for all $x\in\Omega$, the marginal prior variance at any location $x\in\Omega$ is
\[
\mathrm{Var}[F_x]=\phi_{\ell,f,n}(x,x)=\sigma_f^2+\sigma_n^2,
\]
which decomposes into the latent signal variance and the measurement noise variance.

To construct the covariance matrix $\Phi_{\ell,f,n}\in\mathbb{R}^{N\times N}$, we separate signal and noise contributions.
Let $K_{\ell,f}=\sigma_f^2 K_{\ell}$ denote the signal covariance matrix. Then the full covariance matrix on the training set
$X=\{x_1,\dots,x_N\}$ is
\begin{equation*}
\Phi_{\ell,f,n}=K_{\ell,f}+\sigma_n^2 I
=\sigma_f^2 K_{\ell}+\sigma_n^2 I.
\end{equation*}

Similarly, the cross-covariance vector between a test location $x\in\Omega$ and the training points in $X$ is given by
\[
{\rm k}_{\ell,f}(x) = (\phi_{\ell,f,n}(x,x_1),\dots,\phi_{\ell,f,n}(x,x_N))^\intercal 
=
\sigma_f^2\,{\rm k}_\ell(x),
\]
since the noise term $\sigma_n^2\delta_{xx'}$ contributes only to the marginal covariances of the observed data,
and therefore does not affect cross-covariances between $F_x$ and $Y$.

In order to highlight the connection with the deterministic setting, we consider the conditional random variable
$F_x\mid Y=y$, describing the response at a test point $x\in\Omega$ conditioned on the observations $y$.
For simplicity, we denote it by $F_x\mid y$.
Its posterior distribution is Gaussian, with mean and variance given by
\cite{rasmussen2005gaussian,stein1999interpolation}:
\begin{align}\label{eq:posterior}
\mathbb{E}[F_x\mid y]
&=
{\rm k}_{\ell,f}(x)^\intercal \Phi_{\ell,f,n}^{-1}y,\\
\mathrm{Var}[F_x\mid y]
&=
\phi_{\ell,f,n}(x,x)
-
{\rm k}_{\ell,f}(x)^\intercal \Phi_{\ell,f,n}^{-1}{\rm k}_{\ell,f}(x). \nonumber
\end{align}

Factoring out $\sigma_f^2$ in \eqref{eq:posterior}, and letting
$\lambda=\sigma_n^2/\sigma_f^2$ denote the inverse signal-to-noise ratio, the posterior mean simplifies to
\begin{equation}\label{eq:mean-simple}
\mathbb{E}[F_x\mid y]
=
{\rm k}_{\ell}(x)^\intercal (K_{\ell}+\lambda I)^{-1}y,
\end{equation}
while the posterior variance becomes
\begin{equation}\label{eq:post_variance}
\mathrm{Var}[F_x\mid y]
=
\sigma_f^2\Big[\kappa_\ell(x,x)-{\rm k}_{\ell}(x)^\intercal (K_{\ell}+\lambda I)^{-1}{\rm k}_{\ell}(x)\Big]
+\sigma_n^2.
\end{equation}
In this form, the first term represents the uncertainty due to the latent process, whereas the second term corresponds
to the irreducible measurement noise.
This framework allows one to provide predictions together with confidence intervals.
For a fixed confidence level $1-\alpha\in(0,1)$, denoting by $z_{1-\alpha/2}$ the standard normal quantile, we define
\begin{equation*}
\mathrm{CI}^{(f)}_\alpha(x)
=
\mathbb{E}[F_x\mid y]\pm z_{1-\alpha/2}\sqrt{\mathrm{Var}[F_x\mid y]}.
\end{equation*}



To recover the deterministic interpolation setting exactly, we introduce the so-called \emph{smoothed data} \cite{Fasshauer-et-al:2015,Schaback_Wendland_2006}
\[
\hat{y}=(K_{\ell}+\lambda I)^{-1}K_{\ell} \, y.
\]
Then, for any $x\in\Omega$, by the push-through identity \cite[Fact 2.16.16]{Bernstein} we obtain
\begin{align*}
\mathbb{E}[F_x\mid y]
&=
{\rm k}_{\ell}(x)^\intercal (K_{\ell}+\lambda I)^{-1}y
=
s_X^\lambda(x)\\
&=
{\rm k}_{\ell}(x)^\intercal K_{\ell}^{-1}(K_{\ell}+\lambda I)^{-1}K_{\ell} y
=
{\rm k}_{\ell}(x)^\intercal K_{\ell}^{-1}\hat{y}:=S_X(x), 
\end{align*}
\begin{equation}\label{bho}
\mathrm{Var}[F_x\mid y]
=
\sigma_f^2\Big[\kappa_\ell(x,x)-{\rm k}_{\ell}(x)^\intercal K_{\ell}^{-1}{\rm k}_{\ell}(x)\Big]
=
\sigma_f^2 P_{\kappa_\ell,X}(x)^2.
\end{equation}
Hence, in the presence of Gaussian noise, the described Kriging {prediction $S_X(x)$} reduces to interpolating the
smoothed data $\hat{y}$. In this RKHS interpolation  framework, we observe that $\sigma_f^2$ cancels out in the posterior mean in \eqref{eq:mean-simple}. Consequently, $\sigma_f^2$ would be absorbed into the coefficients $\alpha_i$, $i=1,\dots,N$, and can therefore be ignored as it does not affect the approximant itself. However, in the probabilistic interpretation, $\sigma_f^2$ scales the posterior variance and ensures that uncertainty is expressed in the correct physical units.


\begin{remark}[Universal kriging]
Approximation with positive definite kernels is closely related to the \emph{simple Kriging} framework, in which the mean of the (Gaussian) random field is assumed to be zero. 
In contrast, \emph{conditionally positive definite} kernels arise naturally when the mean is unknown but constrained to lie in a prescribed finite-dimensional function space, such as the space of polynomials. 
This distinction mirrors the functional-analytic theory of reproducing kernel Hilbert spaces (RKHSs) versus their conditional counterparts, often referred to as \emph{native spaces} associated with conditionally positive definite kernels.
\end{remark}

\subsection{Non-stationary kernels}

We recall that a kernel $\kappa$ is stationary if it depends only on the difference of its arguments, i.e. it is a translation invariant kernel and there exists a function $\rho$ such that $\kappa(x,x') = \rho(x - x')$.  A special case of stationary kernels are radial kernels for which there exist a so called radial basis function $\varphi: \mathbb{R}_+ \longrightarrow \mathbb{R}$ such that, for all $x,x' \in \Omega$,
\begin{equation*} 
\kappa_\ell(x,x') = \varphi_\ell(\|x-x'\|_2) = \varphi\left(\frac{r}{\ell}\right), \quad r = \|x-x'\|_2.
\end{equation*} 


Stationary radial kernels encode spatially homogeneous correlation structures: the covariance between observations at locations $x$ and $x'$ depends solely on their relative distance, not on their absolute positions. Although this assumption ensures analytical tractability and leads to well-understood theoretical properties \cite{stein1999interpolation,rasmussen2005gaussian,scheuerer2013interpolation}, it restricts the class of functions that can be effectively modelled. 
In particular, a fixed stationary kernel imposes spatially homogeneous regularity: the process has the same degree of smoothness/roughness throught the domain. Consequently, it cannot adapt to spatially varying regularity, and it does not naturally represent localised features such as jumps or discontinuities. {Non-stationary} kernels relax these constraints by allowing the covariance structure to depend explicitly on the locations $x$ and $x'$. 

A straightforward yet meaningful non-stationary kernel is the linear kernel $\kappa(x,x') = x^\intercal x'$, $x,x' \in \Omega$. Moreover several broad classes of non-stationary constructions have been developed in the literature and summarized below. Common approaches introduce spatially varying amplitude on a stationary baseline through a covariance function of the form
\begin{equation*}
\phi_{\ell,f,n}(x,x') = \sigma_f(x) \, \sigma_f(x') \, \kappa_\ell(x,x') + \sigma_n^2 \delta_{xx'},
\end{equation*}
where $\kappa_\ell$ is a stationary kernel and $\sigma_f(\cdot)$ is a positive function encoding local variability \cite{paciorek2006environmetrics,higdon1998processconv}. In literature $\sigma_f(x) \, \sigma_f(x') \, \kappa_\ell(x,x')$ are also known  as \textit{amplitude modulation} kernels. They preserve the correlation structure of $\kappa_\ell$ while modulating the marginal variance:
\[
\mathrm{Var}[F_x] = \phi_{\ell,f,n}(x,x) = \sigma_f^2(x) + \sigma_n^2.
\]
Such kernels are particularly effective for heteroscedastic data where the signal amplitude varies systematically across the domain.

Other classes of non-stationarity kernels can be induced by letting varying the length scale parameter. Among them, we remind  the so-called Gibbs kernel \cite{paciorek2006environmetrics}  defined by
\begin{equation}\label{eq:gibbs}
\kappa_{\ell(x,x')}(x,x') =\sqrt{\frac{2\ell(x)\ell(x')}{\ell(x)^2 + \ell(x')^2}} \, \varphi\!\left( \frac{\|x - x'\|_2}{\ell_{\mathrm{eff}}(x,x')} \right),
\end{equation}
where $\ell: \mathbb{R}^d \longrightarrow (0, +\infty)$ is a spatially varying length scale and $\ell_{\mathrm{eff}}(x,x') = \sqrt{\frac{\ell(x)^2 + \ell(x')^2}{2}}$. This allows the process to exhibit smooth, slowly varying behavior in some regions (large values of $\ell(x)$) and rapid fluctuations in others (small values of $\ell(x)$). 

Moreover, Paciorek and Schervish \cite{paciorek2004nonstationary} introduced a broad class of non-stationary covariance functions of the form
\begin{equation}\label{eq:ps-kernel}
\kappa^{\Sigma}(x,x')
=
\frac{
|\Sigma(x)|^{1/4} |\Sigma(x')|^{1/4}
}{
\big|\tfrac{1}{2}(\Sigma(x)+\Sigma(x'))\big|^{1/2}
}
\,
\varphi_{\Sigma}(x,x')
\end{equation}
where 
\begin{equation*}
    \varphi_{\Sigma}(x,x') = \varphi\!\left(
\sqrt{
(x-x')^\intercal
\Big(\tfrac{1}{2}(\Sigma(x)+\Sigma(x'))\Big)^{-1}
(x-x')
}
\right).
\end{equation*}
Here $\varphi$ is a radial kernel and for $x \in \Omega,$
$\Sigma(x)\in\mathbb{R}^{d\times d}$ is a smoothly varying field of symmetric positive definite matrices. Note that the Gibbs kernel represents a particular case of $\kappa^\Sigma$ in the one-dimensional setting if the Gaussian is chosen as underlying radial function.

Non-stationarity may also be induced through a spatial deformation or a warping map $g: \Omega \to \Omega'$, being $\Omega' \subset \mathbb{R}^d$, leading to kernels of the form \cite{sampson1992nonparametric,schmidt2003deformations}
\begin{equation}\label{eq:warped}
\kappa^g_\ell(x,x') = \varphi_\ell(\|g(x) - g(x')\|_2),\quad x,x' \in \Omega.
\end{equation}
The map $g$ encodes the geometric structure of non-stationarity, effectively stretching or compressing distances in different regions of the domain. This perspective connects naturally to Riemannian geometry, where the kernel is interpreted as being stationary with respect to a spatially varying metric \cite{borovitskiy2020matern}.

In the next section, we show that {variably scaled kernels defined in \cite{bozzini2015interpolation}} naturally introduce non-stationarity.

\section{Variably scaled kernels in a statistical perspective} \label{VSK-GP}
Letting $\psi:\Omega \to \mathbb{R}^q,$ $q \ge 1,$ be a \emph{scaling function}, and  $\Psi(x)=(x,\psi(x)),$  a VSK is defined as 
\begin{equation}\label{eq:vsk}
\kappa^\Psi_\ell(x,x') = \varphi_\ell(\|\Psi(x) - \Psi(x')\|_2), \quad x,x' \in \Omega.
\end{equation}
Since $\Psi$ is injective, the (strictly) positive definiteness is preserved by pullback, and the RKHSs of $\kappa_\ell$ and $\kappa^\Psi_\ell$ are isometrically isomorphic. 
The following geometric property
\begin{equation}\label{eq:dist}
\|{\Psi}(x)-\Psi(x')\|_2^2
\ge \|x-x'\|_2^2,\quad x,x'\in\Omega
\end{equation}
was proven in \cite{bozzini2015interpolation} and will be useful later.
It follows that Euclidean distances never decrease under the map $\Psi.$ This fact  has been used to construct the scaling function $\psi$ in order to improve the stability of the interpolation process (see e.g. \cite{DeMarchi-et-al:2019,rossini2022overview}). 

Another use of $\psi$ is to \textit{mimic} the shape of the target function $f$. {In this case, VSKs significantly improve the recovery quality.} For instance, when $f$ exhibits a jump discontinuity at $x_0$, choosing $\psi$ to be piecewise constant (with a jump at $x_0$) {prevents undesired oscillations in the interpolant near the discontinuity and yields an interpolant that is itself discontinuous, consistently with the target. In this case, we refer to  variably scaled discontinuous kernels (VSDKs) \cite{demarchi2019vsdk,DeMarchi-et-al:2020}. This choice ensures that correlations across the discontinuity are substantially reduced, effectively decoupling the two sides.} {Similarly, for functions with corner-type singularities, $\psi,$ although globally smooth, may be designed to mimic corner behaviour and go rapidly to zero, thus adapting the local correlation length to the reduced regularity of the target $f.$} 

{In general,} points $x$ and $x'$ that are close in Euclidean distance but for which $\psi(x)$ and $\psi(x')$ differ significantly are treated as being far apart in the augmented-dimension space.
In this sense, VSKs provide a structured mechanism for incorporating prior knowledge about the spatial heterogeneity or local irregularity of the target into the approximation model. In the framework of Kriging approximations, this means that we can impose some information about the target to influence the posterior distributions.

{To better understand the implications of this construction, we relate the VSK framework to classical non-stationary kernels and subsequently analyze the associated VSK Kriging variance. }

\subsection{Connections between VSKs and classical non-stationary kernels}\label{sec:theory}

We now establish explicit connections between the VSK construction and the classical non-stationary kernel formulations introduced in the previous section. We observe that VSKs represent a special instance of warped kernels with $\Omega'=\mathbb{R}^{d+q}$ and $g=\Psi$ (see \eqref{eq:warped}). The resulting covariance function 
\begin{equation}\label{eq:vsk-def:cov}
\phi^\Psi_{\ell,f,n}(x,x') = \sigma_f^2 \kappa^{\Psi}_\ell(x,x') + \sigma_n^2 \delta_{xx'},
\end{equation}
is non-stationary in the original coordinates but stationary in the embedded space. 

In the case of the linear kernel, the spatial covariance structure can be enriched by VSKs as 
\begin{equation*}
\phi^\Psi_{f,n}(x,x') = \sigma_f^2\big(x^\intercal x' + \psi(x)^\intercal \psi(x')\big)+\sigma_n^2\delta_{xx'},
\end{equation*}
whose marginal variance is indeed
\[
\mathrm{Var}[F_x \mid y]=\phi^\Psi_{f,n}(x, x)=\sigma_f^2\big(\|x\|_2^2 + \|\psi(x)\|_2^2\big)+\sigma_n^2, \quad x \in \Omega.
\] 

Furthermore, {Gaussian} VSKs are related to amplitude modulation in the sense of the following proposition.

\begin{proposition}[Gaussian VSKs and amplitude modulation]
Let $d,q \geq 1$ and consider the squared exponential (Gaussian) kernel $\kappa_\ell$ built on $\varphi_\ell(r) = \exp\!\left(-\frac{1}{2\ell^2}r^2\right)$. Then, letting $\phi^\Psi_{\ell,f,n}(x,x')=\sigma_f^2 \kappa^{\Psi}_\ell(x,x') + \sigma_n^2 \delta_{xx'}$ for $x,x'\in\Omega$, we have that 
\[
\phi^\Psi_{\ell,f,n}(x,x') = \tilde{\sigma}_f(x)\,\tilde{\sigma}_f(x')\, r^\Psi_\ell(x,x')+ \sigma_n^2 \delta_{xx'},
\]
where
\[
r^\Psi_\ell(x,x') := \exp\!\bigg(\frac{1}{\ell^2} \langle \psi(x), \psi(x')\rangle \bigg) \, \kappa_\ell(x,x'),
\]
and
\[\quad \tilde{\sigma}_f(x) = \sigma_f \exp\!\left(-\frac{1}{2\ell^2}\|\psi(x)\|_2^2\right),
\]
\end{proposition}

\begin{proof}
By definition
\[
\kappa^\Psi_\ell(x,x') = \exp\!\left(-\frac{1}{2\ell^2}\Big(\|x-x'\|_2^2 + \|\psi(x)-\psi(x')\|_2^2\Big)\right),
\]
which leads {directly to the proof}
\begin{align*}
\phi^\Psi_{\ell,f,n}(x,x') = &  \underbrace{\sigma_f \exp\!({-\dfrac{1}{2 \ell} \| \psi(x) \|_2})}_{=:\tilde{\sigma}_f(x)} \, \underbrace{\sigma_f \exp\!({-\dfrac{1}{2 \ell} \| \psi(x') \|_2})}_{=:\tilde{\sigma}_f(x')} \, \underbrace{\exp\!\bigg(\frac{1}{\ell^2} \langle \psi(x), \psi(x')\rangle \bigg) \, \kappa_\ell(x,x')}_{=: r^\Psi_\ell(x,x')}  & \\ &
+\sigma_n^2 \delta_{xx'}.
\end{align*}
\end{proof}

Note that $r^\Psi_\ell$ is a modified kernel that is no longer  stationary but contains a cross-term depending on $\psi(x)$, $\psi(x')$, showing that VSKs encode a richer structure than pure amplitude modulation. 

As far as the connection with Gibbs kernel is concerned, we get the following through local linearization. Some of the proofs of the results reported in this section are shown in Appendix.

\begin{theorem}[Local symmetric expansion of the VSK covariance]\label{prop:vsk-local-symmetric}
{Let us consider a radial VSK as defined in \eqref{eq:vsk} with global length scale $\ell>0$ and 
$\varphi_\ell(r)=\varphi(r/\ell)$ with $\varphi(0)=1$. We assume that $q=1,$ i.e. $\psi:\Omega\to\mathbb{R}$ and that $\psi\in C^1(\Omega).$ Then the VSK covariance $\phi^\Psi_{\ell,f,n}$ defined in \eqref{eq:vsk-def:cov} admits, for all $x,x'\in\Omega$ with $h=x'-x$ satisfying $\|h\|_2=o(1),$ the local approximation}

\begin{equation*}
\phi^\Psi_{\ell,f,n}(x,x')
=
\sigma_f^2\,
\varphi\!\left(
\frac{\sqrt{h^\intercal \,\overline{M}(x,x')\, h}}{\ell}
\right)
+
\sigma_n^2 \delta_{xx'}
+ o(1),\quad x'\to x,
\end{equation*}
where the symmetric metric tensor $\overline{M}(x,x')$ is given by
\begin{align}\label{eq:Mbar-def}
\overline{M}(x,x')
= & \frac{1}{2}(M(x)+M(x')) \\& =
I+\frac{1}{2}\Big(
\nabla\psi(x)\nabla\psi(x)^\intercal
+
\nabla\psi(x')\nabla\psi(x')^\intercal
\Big). \nonumber
\end{align}
\end{theorem}

\begin{corollary}\label{cor:vsk-d1}
In the same hypothesis of Theorem \ref{prop:vsk-local-symmetric} and assuming that $d=1$, for $x,x'\in\Omega$ with $|x'-x|=o(1)$,
\[
\phi^\Psi_{\ell,f,n}(x,x')
=
\sigma_f^2\,
\varphi\!\left(
\frac{|x-x'|}{\tilde{\ell}^{\Psi}(x,x')}
\right)
+
\sigma_n^2 \delta_{xx'}
+
{o(1),\quad x'\to x,}
\]
where
\[
\ell^{\Psi}(x)
=
\frac{\ell}{\sqrt{1+(\psi'(x))^2}},
\]
and 
\[
\tilde{\ell}^{\Psi}(x,x')
=
\sqrt{\frac{\ell^{\Psi}(x)^2+\ell^{\Psi}(x')^2}{2}}.
\]
\end{corollary}

Theorem~\ref{prop:vsk-local-symmetric} shows that the VSK covariance is locally governed by a locally varying metric
\[
M(x)=I+\nabla\psi(x)\nabla\psi(x)^\intercal
\]
on the input space $\Omega$. Equivalently, the VSK kernel is locally stationary with respect to the
anisotropic norm
\[
\|h\|_{M(x)}=\sqrt{h^\intercal M(x)h}.
\]
This places VSK covariances within the general class of \emph{locally stationary Gaussian processes}
\cite{silverman1957, priestley1965}, where stationarity holds only infinitesimally and the local
spectral properties vary smoothly with the input location.

In the one-dimensional case ($d=1$), Corollary~\ref{cor:vsk-d1} shows that the local metric reduces to a scalar factor that indicates a similar behaviour to Gibbs kernels (see \eqref{eq:gibbs}), with a locally varying length-scale parameter. The only difference is that the multiplicative factor, that tends to $1$ as $x' \to x$, is absent. Note that the effective length scale $\ell^\Psi(x)$ is not prescribed directly, but emerges from the geometry of the embedding via the gradient of the scaling function $\psi$.

For $d\ge 1$, we can derive the following result, which shows that the VSK covariance is locally equivalent, up to higher-order
terms, to a Paciorek--Schervish kernel.

\begin{corollary}[Local equivalence of VSKs with Paciorek--Schervish kernels]\label{prop:vsk-ps}
Let $\phi^\Psi_{\ell,f,n}$ be a VSK covariance as in
Theorem~\ref{prop:vsk-local-symmetric}, and define the matrix field
\begin{equation}\label{eq:Sigma-vsk}
\Sigma(x)
=
\ell^2
\big(I+\nabla\psi(x)\nabla\psi(x)^\intercal\big)^{-1}= {\ell^2 M(x)^{-1}}.
\end{equation}
Then, for $x,x'\in\Omega$ with $h=x'-x$ satisfying $\|h\|_2=o(1)$,
the quadratic form appearing in the Paciorek--Schervish kernel \eqref{eq:ps-kernel} satisfies
\begin{align}\label{eq:ps-local-limit}
& (x-x')^\intercal
\Big(\tfrac{1}{2}(\Sigma(x)+\Sigma(x'))\Big)^{-1}
(x-x')
=
\frac{1}{\ell^2}\,
h^\intercal \overline{M}(x,x') h
+
{o(\|h\|_2^2),\quad h\to 0},
\end{align}
where
\[
\overline{M}(x,x')
=
I+\frac{1}{2}\big(
\nabla\psi(x)\nabla\psi(x)^\intercal
+
\nabla\psi(x')\nabla\psi(x')^\intercal
\big).
\]
\end{corollary}

We end up this section by pointing out some important facts about the structure of the covariance kernels arising when approximating discontinuous targets by means of piecewise constant scaling functions. 
\begin{remark}[Discontinuous scaling functions]\label{rem:decoupling}
    For piecewise constant $\psi$ (as used for discontinuities), the gradient $\nabla\psi$ is zero almost everywhere except at the discontinuity. In this case:
\begin{itemize}
\item Within each smooth region, we obtain a standard stationary kernel.
\item Across the discontinuity, when $q=1$, the absolute difference $|\psi(x) - \psi(x')|$ is large even for small $\|x - x'\|_2$, effectively setting $\|\Psi(x) - \Psi(x')\|_2 \gg \|x - x'\|_2$ and causing $\kappa_\psi(x,x') \approx 0$. This mechanism  decouples the correlation structure across the jump. 
\end{itemize}
\vskip 0.3cm
This behaviour is not captured by the analysis developed in this section, which concerns the relationship with Gibbs and Paciorek–Schervish kernels in the case of smoothly varying scaling functions. It instead highlights a distinctive feature of VSKs, namely their ability to naturally encode discontinuities as well as smooth variations in regularity.
\end{remark}

\subsection{A note on the VSK Kriging variance}

The purpose of this section is to investigate the behaviour of the power function (see \eqref{bho}) in the VSK interpolation setting and later the Kriging variance in case of noisy observations. 
Our result provides upper and lower bounds for the VSK power function in terms of the original kernel $\kappa_\ell$. 
\begin{proposition}\label{prop:lower-weyl}
Assume that $\kappa_\ell$ and $\kappa_\ell^\Psi$ are such that $\forall x \in \Omega$ 
$\|{\rm k}_{\ell}^\Psi(x)\|_2 \le \|{\rm k}_{\ell}(x)\|_2$, $\kappa_\ell^\Psi(x,x) = \kappa_\ell(x,x)$, and
\begin{equation*} 
\lambda_{\min}({ K}_{\ell}^\Psi) \ge \lambda_{\min}({ K}_{\ell}), \qquad
\lambda_{\max}({ K}_{\ell}^\Psi) \le \lambda_{\max}({ K}_{\ell}).
\end{equation*}
Then
\begin{equation}\label{eq:lower-weyl}
\kappa_\ell(x,x) - \frac{\|{\rm k}_{\ell}(x)\|_2^2}{\lambda_{\min}({ K}_{\ell})} \le P_{\kappa_\ell^\Psi, X}^2(x) \le \kappa_\ell(x,x) - \frac{\|{\rm k}^\Psi_{\ell}(x)\|_2^2}{\lambda_{\max}({ K}_{\ell})}.
\end{equation}
\end{proposition}

\begin{proof}
Since $\kappa_\ell^\Psi(x,x) = \kappa_\ell(x,x)$, by Rayleigh--Ritz we have
\begin{align*}
P_{\kappa_\ell^\Psi, X}^2(x) &= \kappa_\ell^\Psi(x,x) - {\rm k}_{\ell}^\Psi(x)^\intercal ({ K}_{\ell}^\Psi)^{-1} {\rm k}_{\ell}^\Psi(x) \\
&\ge \kappa_\ell(x,x) - \frac{\|{\rm k}_{\ell}^\Psi(x)\|_2^2}{\lambda_{\min}({ K}_{\ell}^\Psi)} \\
&\ge \kappa_\ell(x,x) - \frac{\|{\rm k}_{\ell}(x)\|_2^2}{\lambda_{\min}({ K}_{\ell})},
\end{align*}
and similarly,
\begin{align*}
P_{\kappa_\ell^\Psi, X}^2(x) &= \kappa_\ell^\Psi(x,x) - {\rm k}_{\ell}^\Psi(x)^\intercal ({K}_{\ell}^\Psi)^{-1} {\rm k}_{\ell}^\Psi(x)  \\
&\le \kappa_\ell(x,x) - \frac{\|{\rm k}_{\ell}^\Psi(x)\|_2^2}{\lambda_{\max}({ K}_{\ell}^\Psi)} \\
&\le \kappa_\ell(x,x) - \frac{\|{\rm k}_{\ell}^\Psi(x)\|_2^2}{\lambda_{\max}({ K}_{\ell})}.
\end{align*}
\end{proof}

The hypotheses in Proposition \ref{prop:lower-weyl} are satisfied, e.g., for Gaussian and Matérn $C^0$ kernels. Moreover {thanks to \eqref{eq:dist}}, many commonly used radial strictly positive definite kernels, such as the Matérn family, inverse generalized multiquadric, and Wendland's kernels, satisfy the entry-wise inequality
\[
(K^\Psi_\ell)_{ij} \le ({ K}_{\ell})_{ij}, \qquad i,j = 1,\dots,n.
\]
Since both $K^\Psi_\ell$ and ${ K}_{\ell}$ have nonnegative entries, ${K}_{\ell} - K^\Psi_\ell$ is elementwise nonnegative. By the monotonicity of the spectral radius for nonnegative matrices \cite{berman1979nonnegative}, we have
\begin{equation*}
\lambda_{\max}(K^\Psi_\ell) \le \lambda_{\max}({K}_{\ell}).
\end{equation*}
Consequently, the upper bound in \eqref{eq:lower-weyl} also holds for the aforementioned classes of kernels.

The result shown in Proposition \ref{prop:lower-weyl} holds true also when considering the matrix $K_{\ell}+\lambda I$ and $K^{\Psi}_{\ell}+\lambda I$; indeed the effect of adding $\lambda I$ is that of shifting both spectra by a positive scalar $\lambda$.

\section{Numerical examples}\label{sec:examples}

In this section, we show various computational experiments to analyze the behavior of VSKs in the context of Gaussian processes. 

\begin{itemize}
    \item 
    In Section \ref{sec:results_discontinuous}, we test VSKs in the reconstruction of a one-dimensional function that presents a jump discontinuity, varying the number of nodes and the observation noise.
    \item 
    In Section \ref{sec:results_weierstrass}, we reconstruct a two-dimensional truncated Weierstrass function, and we analyze the effectiveness of VSKs varying the scaling function.
    \item 
    In Section \ref{sec:results_corner}, we consider a target function presenting a corner point, and test the effectiveness of our approach in including prior information in the model.
    \item 
    In Section \ref{sec:results_others}, following the theoretical analysis in Section \ref{sec:theory}, we perform some numerical tests to compare VSKs and \textit{corresponding} Gibbs/Paciorek kernels.
\end{itemize}

The experiments are carried out in \textsc{Matlab} on a 11th Gen Intel(R) Core(TM) i7-1165G7 @ 2.80GHz. The hyperparameters for both the classical kernels and VSKs, unless otherwise stated, are \emph{optimized} via Maximum Likelihood Estimation (MLE) \cite{McCourt2017917}  available in the {\tt fitrgp.m} routine. 

To assess the performance of the methods used to approximate the underlying signal $f$, letting $\Xi=\{\xi_1,\dots,\xi_M\}$ be a set of evaluation points in $\Omega$, we will compute the Root Mean Square Error (RMSE) and Maximum Absolute Error (MAE) (cf. \eqref{eq:mean-simple})
\begin{equation*}
\mathrm{RMSE} =
\sqrt{\frac{1}{M}\sum_{i=1}^{M}\left(\mathbb{E}[F_{\xi_i}\mid y] - f(\xi_i)\right)^2},\quad \mathrm{MAE} =
\max_{1 \le i \le M} \left|\mathbb{E}[F_{\xi_i}\mid y] - f(\xi_i)\right|.
\end{equation*}
Moreover, we will consider the posterior variance (cf. \eqref{eq:post_variance}) and calculate
\begin{equation*}
    \mathrm{avg\;std}=\frac{1}{M}\sum_{i=1}^{M}\mathrm{Var}[F_{\xi_i}\mid y],\quad  \mathrm{max\;std}=\max_{1 \le i \le M} \mathrm{Var}[F_{\xi_i}\mid y].
\end{equation*}

\subsection{Approximating a discontinuous function}\label{sec:results_discontinuous}

We consider the target function defined on $\Omega=[0,1]$ as
\[
f(x)
=
\begin{cases}
\dfrac{\cos(14\pi(x+0.5))}{2x+0.5} + (x-0.5)^4+3, & x<0.5,\\[1ex]
\dfrac{\cos(14\pi(x+0.5))}{2x+0.5} + (x-0.5)^4, & x\ge 0.5,
\end{cases}
\]
which exhibits a jump discontinuity at $x_0=0.5$.  For the VSK case, we take the scaling function
\[
\psi(x)=\mathbf{1}_{[x_0,1]}(x),
\]
so that we impose information on the covariance kernel about the discontinuity location.

\subsubsection{Prior-driven behavior in the variably scaled setting}

We first investigate the behavior of the proposed VSK model in a severely data-limited regime, with fixed hyperparameters and a small number of training points, by considering a Matérn $C^2$ kernel. The goal of this experiment is to isolate the effect of the kernel prior, independently of any data-driven hyperparameter tuning. We fix the parameters $\ell=1$, $\sigma_f=8$, $\sigma_n=0$, (interpolation case), and we compare the classical stationary process model with its VSK counterpart. The training set consists of $N=6$ uniformly distributed points in $\Omega$, and predictions are evaluated on a grid of $500$ equispaced points. 

Figure~\ref{fig:approx_jump}  shows the resulting posterior mean with corresponding confidence intervals (related to the posterior variance) in the stationary and VSK setting. Although the number of training nodes is limited, it still allows us to draw some simple but useful observations. Indeed, note that in the VSK case the posterior variance is significantly higher in a neighborhood of $x_0=0.5$, reflecting the information encoded in the embedded space provided by the chosen $\psi$. This is also reflected in Figure \ref{fig:real_jump}, where several realizations from the prior and posterior processes are shown. In the VSK framework, the jump location is explicitly incorporated at the prior stage. In Table \ref{tab:comparison_fixed_sparse}, some metrics further highlight the larger uncertainty provided by VSKs, especially nearby $x_0$, and similar overall accuracy. Finally, in Figure \ref{fig:cov_jump} we display the covariance matrices corresponding to the stationary and the VSK framework, where the decoupling effect provided by the discontinuous covariance is evident (see Remark \ref{rem:decoupling}).

\begin{figure}[H] 
\centering 
\includegraphics[width=0.35\linewidth]{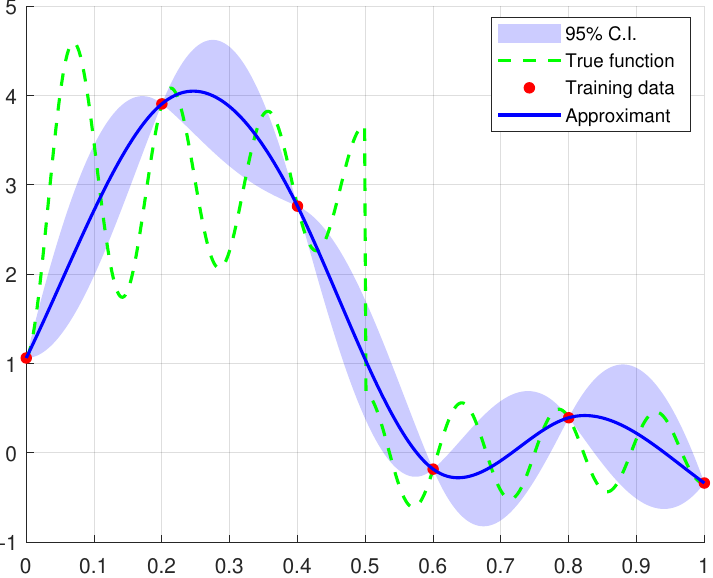}
\hskip 0.5 cm
\includegraphics[width=0.35\linewidth]{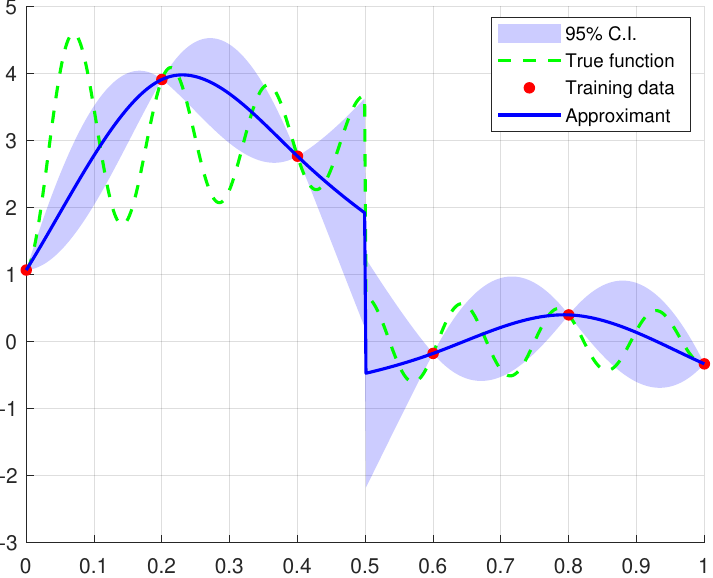} 
\caption{Stationary case (left) and VSK approach (right). The training data consist of $N=6$ equispaced nodes.} \label{fig:approx_jump} 
\end{figure}

\begin{figure}[H] 
\centering 
\includegraphics[width=0.7\linewidth]{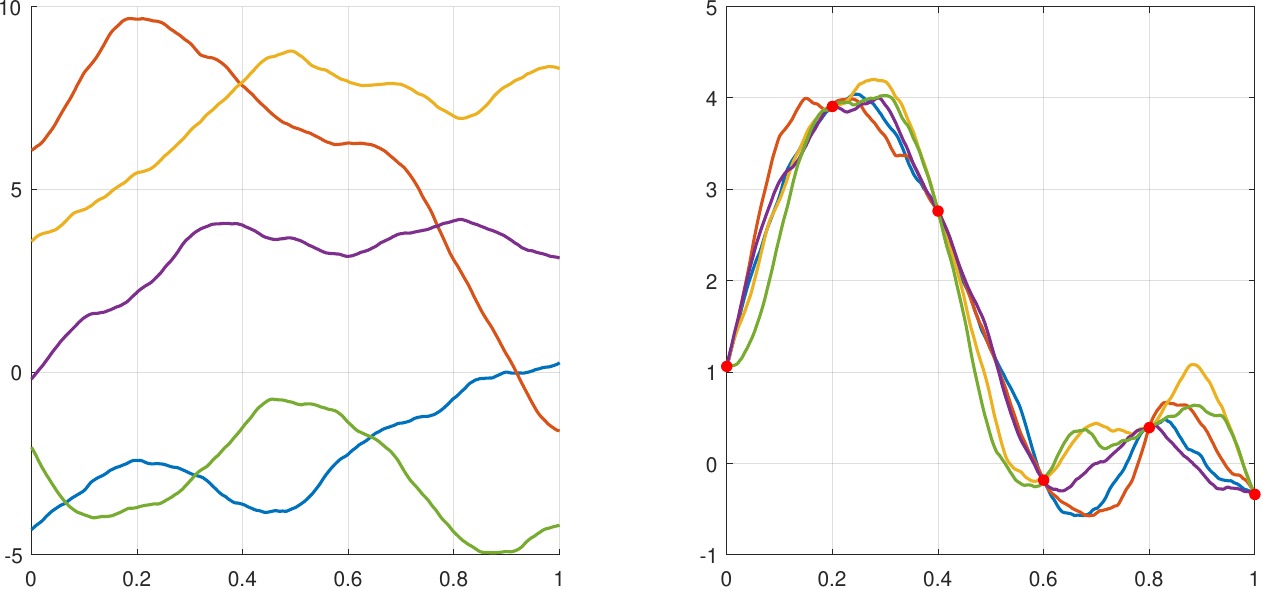} 
\includegraphics[width=0.7\linewidth]{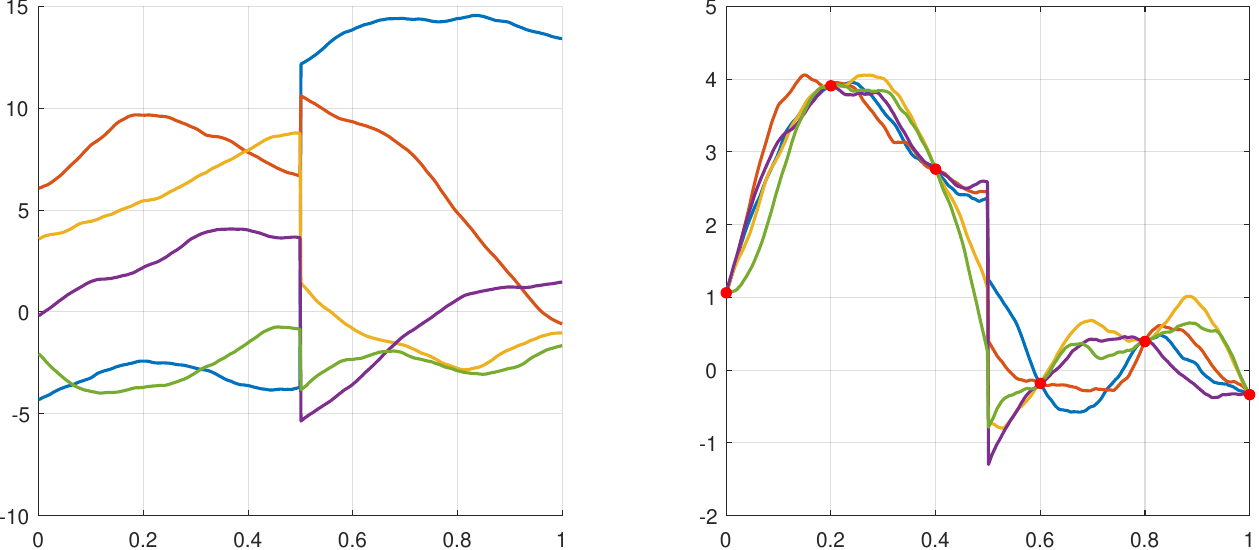}
\caption{Realization of the GPs using the prior (left) and posterior (right) covariance: stationary case (top) and VSK approach (bottom).} 
\label{fig:real_jump} 
\end{figure}

\begin{figure}[H] 
\centering 
\includegraphics[width=0.35\linewidth]{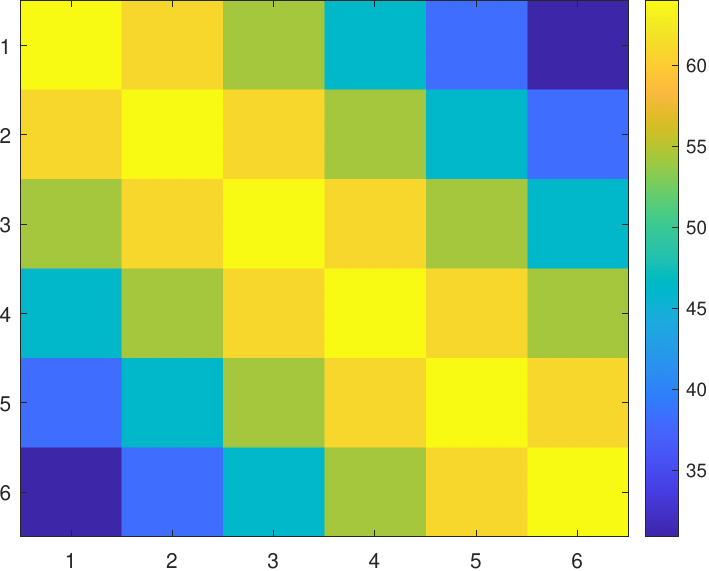} 
\hskip 0.5 cm
\includegraphics[width=0.35\linewidth]{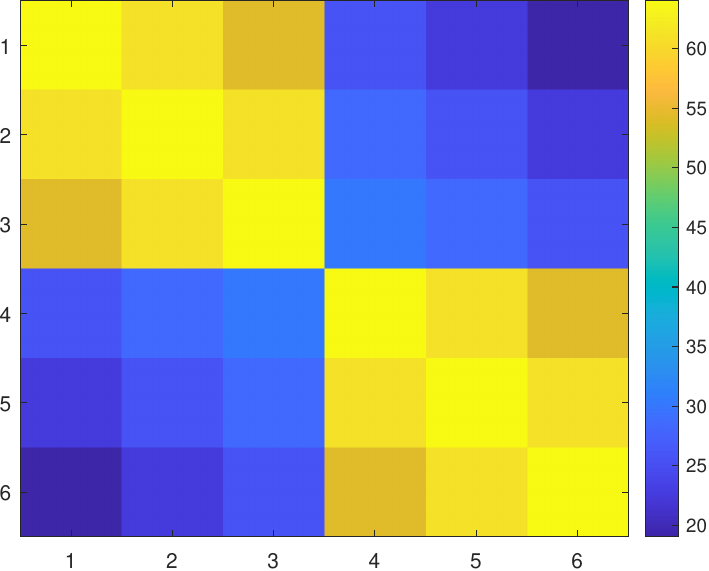}
\caption{Training matrices in the stationary (left) and VSK (right) case.} 
\label{fig:cov_jump} 
\end{figure}

\begin{table}[H]
\centering
\caption{Comparison between the stationary Matérn $C^2$ kernel and the VSK kernel with fixed hyperparameters $\ell=1$, $\sigma_f=8$, $\sigma_n=0,$ and $N=6$ uniform nodes.}
\label{tab:comparison_fixed_sparse}
\begin{tabular}{lcc}
\toprule
 & \textbf{Standard} & \textbf{VSK} \\
\midrule
RMSE                     & $9.1728\times 10^{-1}$ & $\mathbf{8.4409\times 10^{-1}}$ \\
Maximum error            & $2.6058\times 10^{0}$ & $\mathbf{2.3776\times 10^{0}}$ \\
Average posterior std.   & $\mathbf{2.3022\times 10^{-1}}$ & $2.8207\times 10^{-1}$ \\
Maximum posterior std.   & $\mathbf{3.7459\times 10^{-1}}$ & $8.7057\times 10^{-1}$ \\
\bottomrule
\end{tabular}
\end{table}

\subsubsection{Comparison with MLE optimized hyperparameters}

Here, we take the same target function, domain and scaling map as in the previous experiment, but we experiment with a Matérn $C^4$ kernel where hyperparameters are selected by maximizing the GP marginal likelihood. The objective of this experiment is to assess whether the geometric advantages of VSKs persist when the models are tuned in a data-driven way. 

In the following, we consider $N=\{10+20j\:|\:j=1,\dots,39\}$ Halton points. Moreover, a Gaussian observation noise with standard deviation $0.25$ is added to the data. 
In Figure \ref{fig:jump_convergence}, we show the  obtained results.  

\begin{figure}[h] 
\centering 
\includegraphics[width=0.7\linewidth]{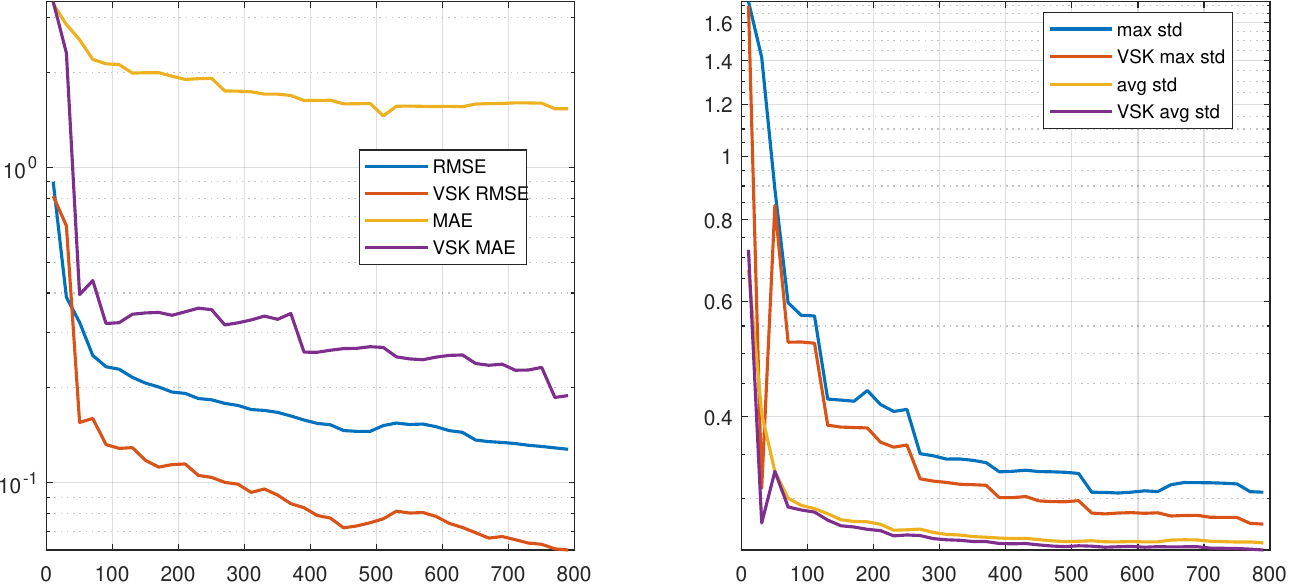} 
\caption{Approximation errors (left) and uncertainty (right) varying $N$ from $10$ to $790$ Halton nodes.} 
\label{fig:jump_convergence} 
\end{figure}

First, we observe that the variably scaled setting leads to more accurate reconstructions and faster error decay than stationary models. Moreover, we note that the standard deviation related to VSK models is analogous to stationary approaches for \textit{small} values of $N$, and then drops as $N$ becomes \textit{large enough}. In Figure \ref{fig:approx_jump_mle_50}, we visualize the reconstructions obtained with $N=27$, while in Figure \ref{fig:approx_jump_mle_250} the results with $N=81$ are depicted. Further details are reported in Tables \ref{tab:comparison_N50} and \ref{tab:comparison_N250}. When only limited information is provided, as in Figure \ref{fig:approx_jump_mle_50}, the prior allows a higher uncertainty near the jump point. Then, as soon as a sufficient amount of observations are provided, the encoded prior matches the available observations, and so it significantly reduces the posterior variance near the jump in both cases.

\begin{figure}[H] 
\centering 
\includegraphics[width=0.35\linewidth]{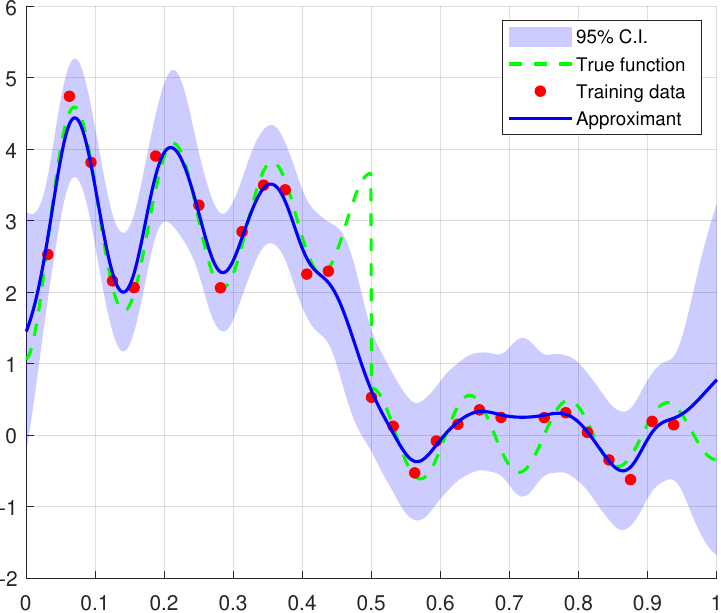} 
\hskip 0.5 cm
\includegraphics[width=0.35\linewidth]{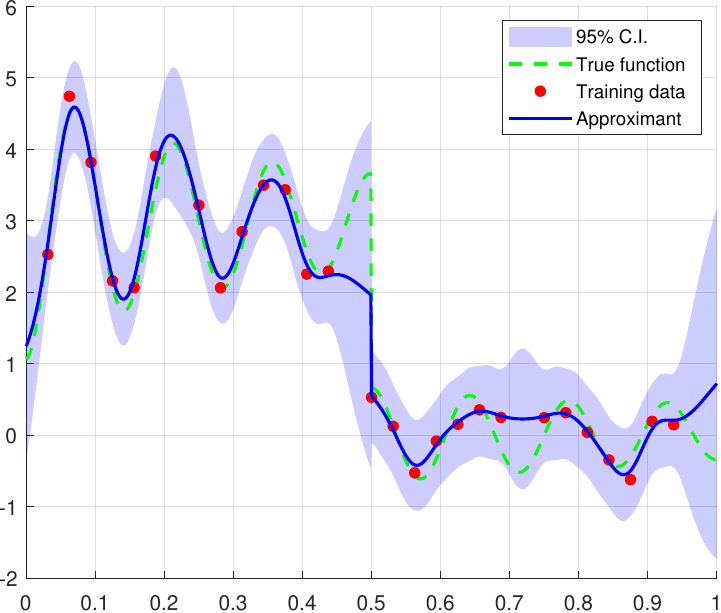} 
\caption{Stationary case (left) and VSK approach (right). The training data consist of $N=27$ Halton nodes.} \label{fig:approx_jump_mle_50} 
\end{figure}

\begin{figure}[H] 
\centering 
\includegraphics[width=0.35\linewidth]{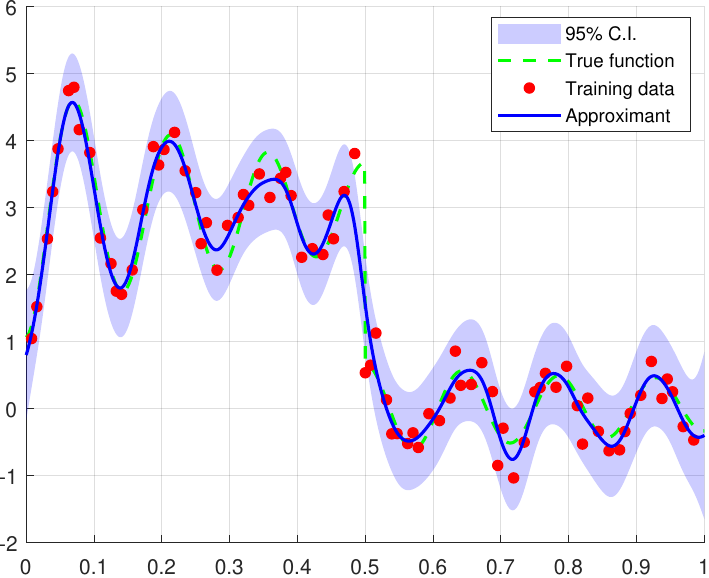} 
\hskip 0.5 cm
\includegraphics[width=0.35\linewidth]{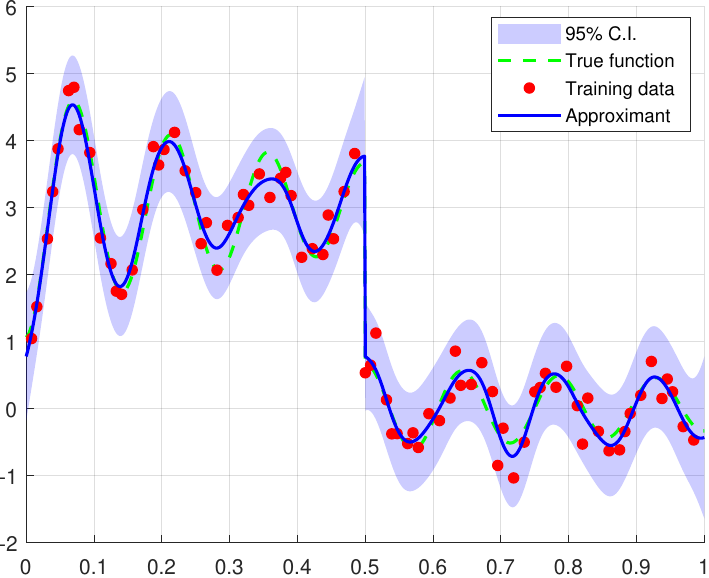} 
\caption{Stationary case (left) and VSK approach (right). The training data consist of $N=81$ Halton nodes.} \label{fig:approx_jump_mle_250} 
\end{figure}

\begin{table}[H]
\centering
\caption{Comparison between stationary kernel and VSK for \(N=27\) training Halton points (MLE-tuned hyperparameters).}
\label{tab:comparison_N50}
\begin{tabular}{lcc}
\hline
\textbf{Metric} & \textbf{Standard} & \textbf{VSK} \\
\hline
Length scale $\ell$      & 0.0650 & 0.0627 \\
Signal variance $\sigma_f$ & 1.4827 & 1.4956 \\
Noise variance $\sigma_n$  & 0.3238 & 0.2428 \\
RMSE                     & $5.3569\times10^{-1}$ & $3.6767\times10^{-1}$ \\
Max error                & $3.0195\times10^{0}$ & $1.7015\times10^{0}$ \\
Average std              & $4.6966\times10^{-1}$ & $4.0266\times10^{-1}$ \\
Max std                  & $1.2539\times10^{0}$ & $1.2502\times10^{0}$ \\
\hline
\end{tabular}
\end{table}

\begin{table}[H]
\centering
\caption{Comparison between stationary kernel and VSK for \(N=81\) training Halton points (MLE-tuned hyperparameters).}
\label{tab:comparison_N250}
\begin{tabular}{lcc}
\hline
\textbf{Metric} & \textbf{Standard} & \textbf{VSK} \\
\hline
Length scale $\ell$      & 0.0614 & 0.0694 \\
Signal variance $\sigma_f$ & 1.6096 & 1.6613 \\
Noise variance $\sigma_n$  & 0.3269 & 0.3321 \\
RMSE                     & $2.4512\times10^{-1}$ & $1.3972\times10^{-1}$ \\
Max error                & $2.0480\times10^{0}$ & $4.1026\times10^{-1}$ \\
Average std              & $3.8423\times10^{-1}$ & $3.8782\times10^{-1}$ \\
Max std                  & $6.3757\times10^{-1}$ & $6.1732\times10^{-1}$ \\
\hline
\end{tabular}
\end{table}
\subsection{Reconstruction of a truncated Weierstrass function with increasing VSK complexity}\label{sec:results_weierstrass}

We consider the reconstruction of a two-dimensional Weierstrass-type function, a classical example of a highly irregular yet continuous signal, used here to further assess the ability of VSKs to incorporate a strong prior.

Let $\Omega=[0,1]^2$. The target function is defined as
\begin{equation*}
f(x_1,x_2)
=
\sum_{k=0}^{K}
a^k
\cos\!\big(\pi b^k x_1\big)
\cos\!\big(\pi b^k x_2\big),
\qquad
a=0.5,\quad b=3,\quad K=12,
\end{equation*}
which corresponds to a truncated two-dimensional Weierstrass function.
Despite the truncation, the resulting surface exhibits fine-scale oscillations and limited regularity. We take {$N=25$ uniform training} points ($5\times 5$ grid) in $\Omega$, while predictions are evaluated on a uniform $50\times50$ grid. No observation noise is considered. We employ the Matérn $C^0$ kernel and estimate all hyperparameters $\ell,\sigma_f,\sigma_n$ via MLE.

To investigate the effect of the scaling map in the VSK approximation, we consider a family of scaling functions of increasing complexity. Specifically, for a given truncation level $K_{\mathrm{VSK}}\in\{0,\dots,12\}$, we define
\begin{equation*}
\psi^{K_{\mathrm{VSK}}}(x_1,x_2)
=
\sum_{k=0}^{K_{\mathrm{VSK}}}
a^k
\cos\!\big(\pi b^k x_1\big)
\cos\!\big(\pi b^k x_2\big).
\end{equation*}
For the case $K_{\mathrm{VSK}}=0$ we set $\psi^{K_{\mathrm{VSK}}}(x_1,x_2)\equiv 0$, which corresponds to a stationary kernel, while by increasing values of $K_{\mathrm{VSK}}$ we progressively incorporate higher-frequency components of the target function into the VSK prior. Note that the extreme case $K_{\mathrm{VSK}}=K$ yields $\psi=f$.

In Figure \ref{fig:weierstrass_convergence}, we report the metrics obtained by varying $K_{\mathrm{VSK}}\in\{0,\dots,12\}$, and in Figure \ref{fig:approx_weierstrass} we display the results achived for some values of $K_{\mathrm{VSK}}$. Moreover, VSKs show robustness against \textit{poor} priors, a fact observed in previous works in the framework of kernel interpolation \cite{DeMarchi-et-al:2020}.

\begin{figure}[H] 
\centering 
\includegraphics[width=0.7\linewidth]{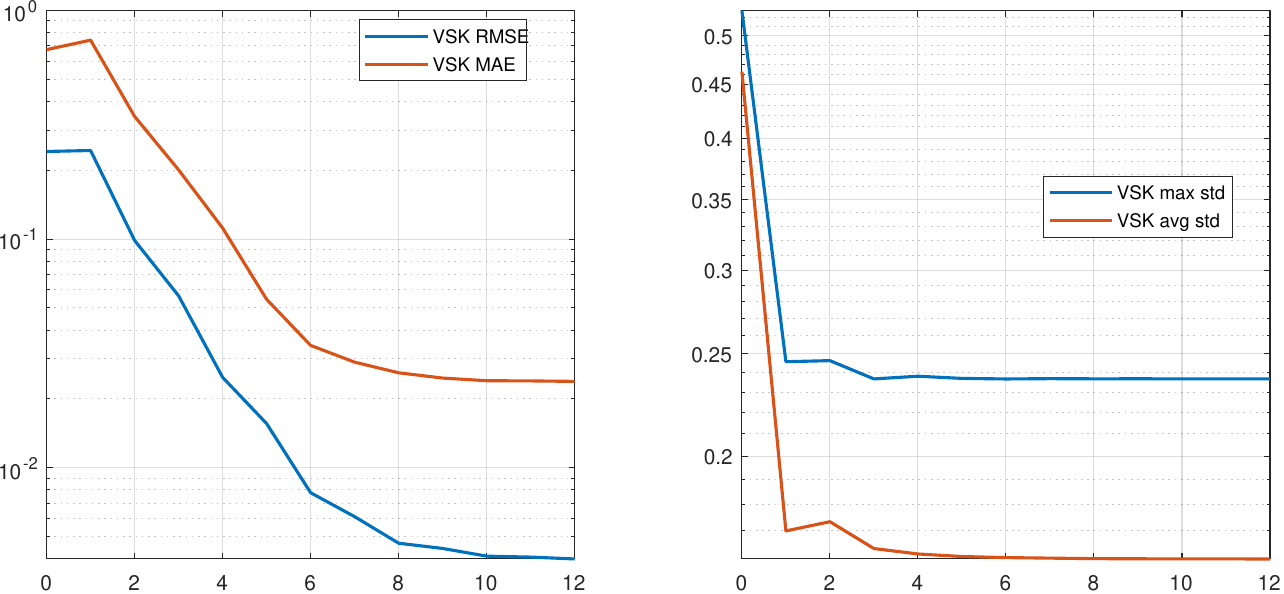} 
\caption{Approximation errors (left) and uncertainty (right) varying $K_{\rm VSK}$ from $0$ (stationary) to $12$.} 
\label{fig:weierstrass_convergence} 
\end{figure}

\begin{figure}[H] 
\centering 
$
\begin{array}{c}
\includegraphics[width=0.35\linewidth]{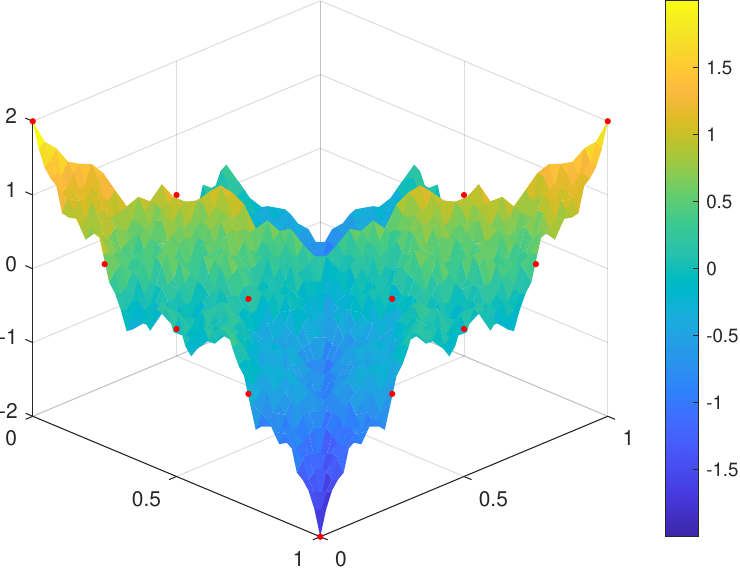}
\includegraphics[width=0.35\linewidth]{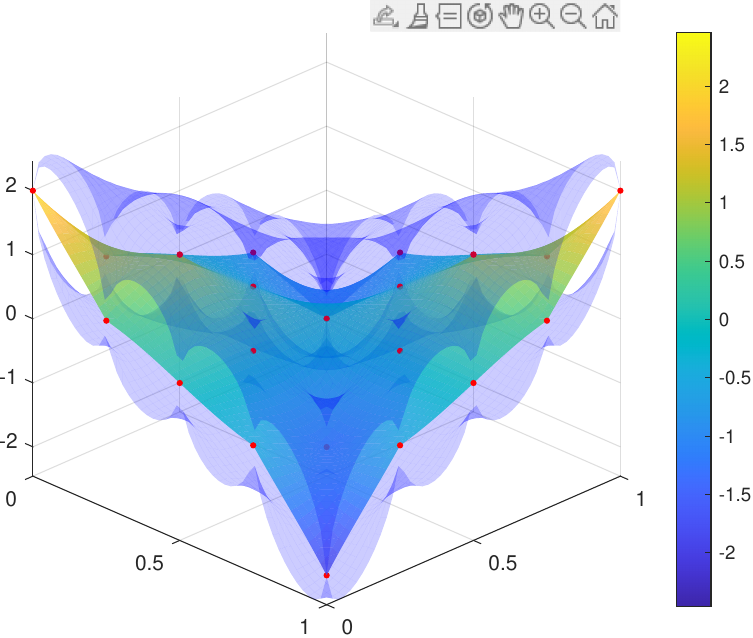}\\
\includegraphics[width=0.35\linewidth]{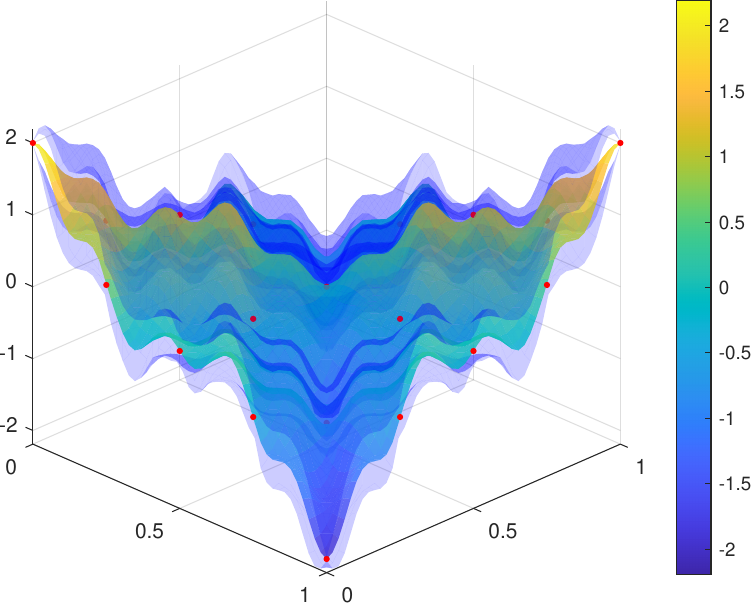}
\includegraphics[width=0.35\linewidth]{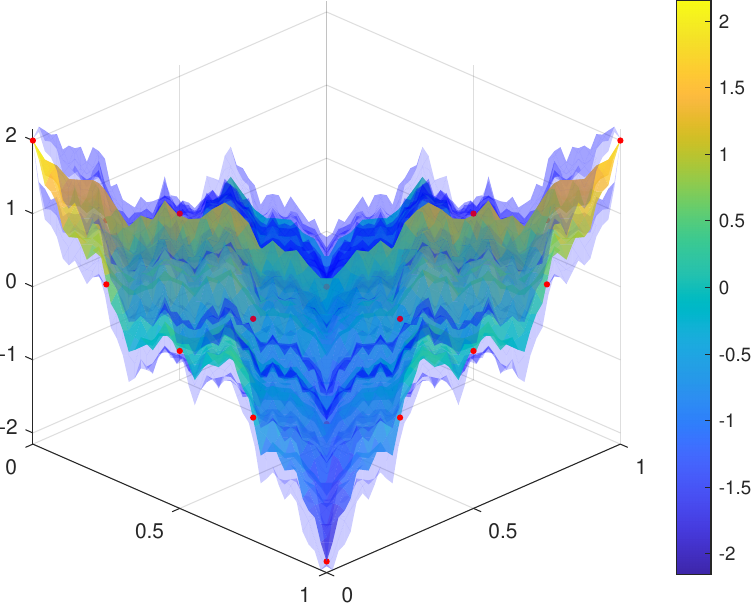}
\includegraphics[width=0.35\linewidth]{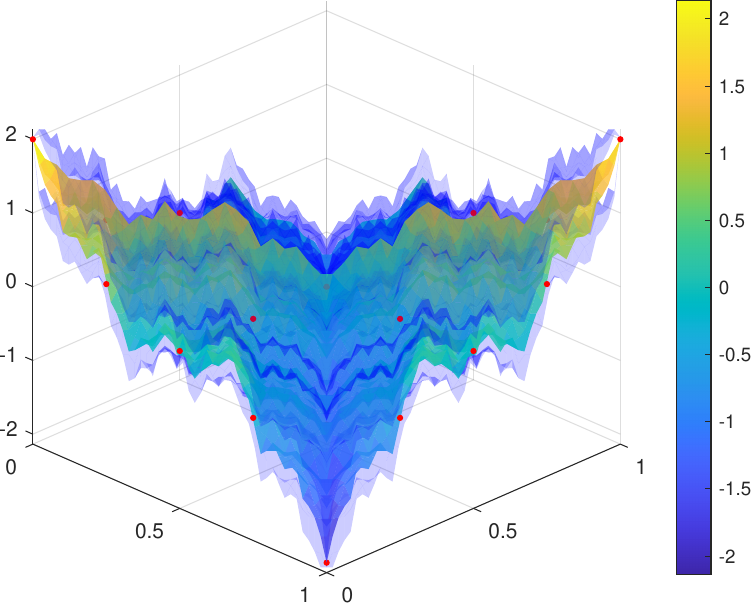}
\end{array}
$
\caption{Approximation of the Weierstrass function, with confidence intervals. Top: the target function (left) and the stationary reconstruction (right). Bottom: VSK reconstructions with $K_{\rm VSK}=2$ (left), $K_{\rm VSK}=4$ (center) and $K_{\rm VSK}=12$ (right).}
\label{fig:approx_weierstrass} 
\end{figure}

\subsection{Approximating a function with a corner}\label{sec:results_corner}

We consider the target function defined on $\Omega=[0,1]$ as
\[
f(x) =
\begin{cases}
\exp\!\left(-\dfrac{1}{2}\left(5(2x-1)-0.5\right)^2\right), & x \ge 0.5, \\[6pt]
\exp\!\left(-\dfrac{1}{2}\left(5(2x-1)+0.5\right)^2\right), & x < 0.5.
\end{cases}
\]
which has a corner point at $x_0=0.5$. We take a Gaussian kernel and, for the VSK case, the scaling function
\begin{equation*}
\psi(x)=\left\{
\begin{array}{ll}
1-\dfrac{3}{2}\dfrac{|x-x_0|}{R}+\dfrac{1}{2}\dfrac{|x-x_0|^3}{R^3}, & |x-x_0|<R,\\
0, & \mathrm{otherwise,}
\end{array}
\right.
\end{equation*}
with $R=0.5$. This scaling function was proposed in \cite{Rossini:2018} for approximating targets with discontinuous  gradients.

In the following, we consider $N=\{11+20j\:|\:j=1,\dots,39\}$ uniform points. With this choice of nodes, note that the corner point $x_0$ is always sampled in the training set. 
In Figure \ref{fig:cusp_convergence}, we show the {obtained} results. 

\begin{figure}[H] 
\centering 
\includegraphics[width=0.7\linewidth]{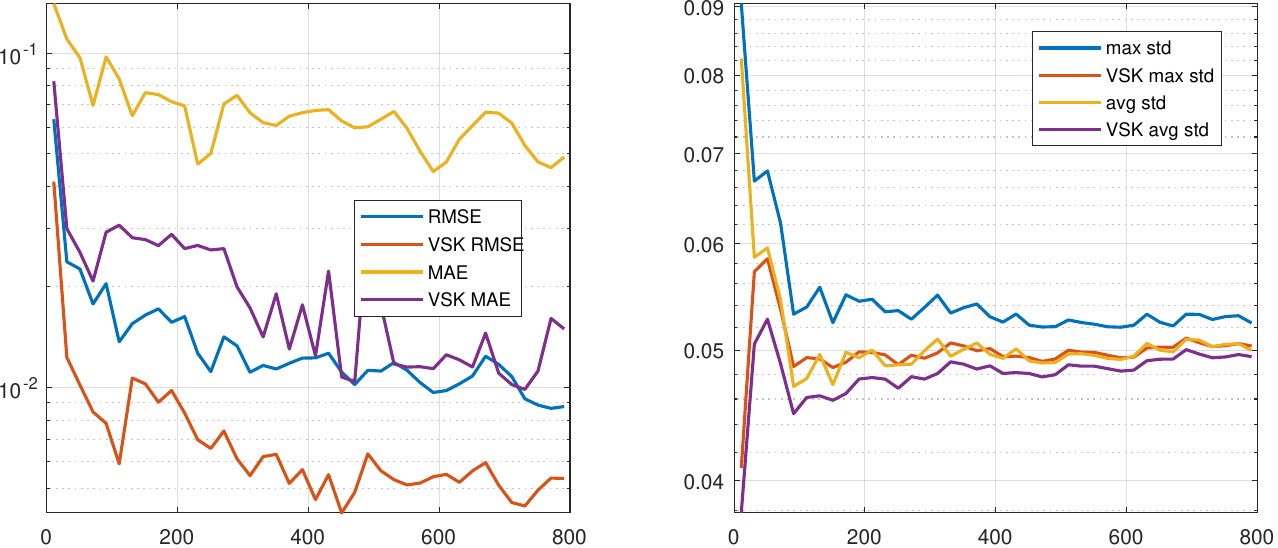} 
\caption{Approximation errors (left) and uncertainty (right) varying $N$ from $11$ to $791$ Halton nodes.}  
\label{fig:cusp_convergence} 
\end{figure}

As in the previous tests, the VSK setting provides more accurate reconstructions than the stationary case. As far as the uncertainty is concerned, we observe a different phenomenon with respect to the jump function case. Here, VSKs are less uncertain than stationary models, especially if few data are considered. This fact can be explained as follows. The corner point is always sampled in the considered training datasets. Thanks to its prior information encoded by $\psi$, VSK models immediately \textit{believe} the corner observation. On the other hand, stationary models need more evidence to stop treating the corner node as a noisy observation. To better highlight this aspect, in Figure \ref{fig:approx_cusp} we show the reconstructions obtained for $N=21$. Furthermore, in Figure \ref{fig:approx_cusp_even}, we show the approximants produced by setting $N=20$, where the effect of the prior is evident also in the case where the corner node is not sampled. The covariance matrices corresponding to the approximations shown with $N=20,21$ are reported in Figure \ref{fig:covariance_cusp}.

\begin{figure}[H] 
\centering 
\includegraphics[width=0.35\linewidth]{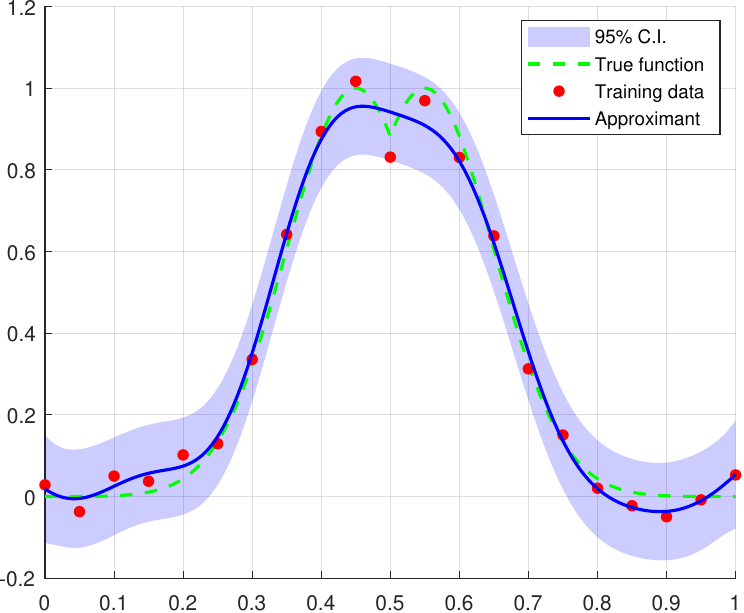}
\hskip 0.5 cm
\includegraphics[width=0.35\linewidth]{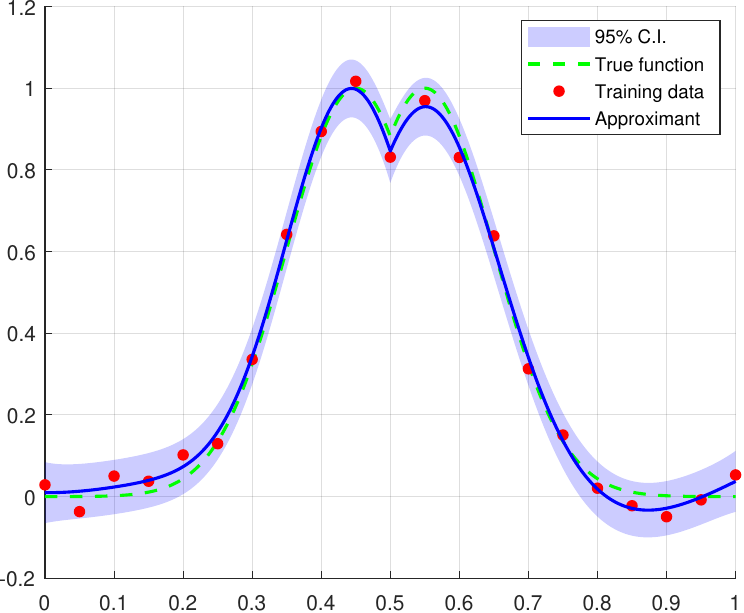} 
\caption{Stationary case (left) and VSK approach (right). The training data consist of $N=21$ uniform nodes.} \label{fig:approx_cusp} 
\end{figure}

\begin{figure}[H] 
\centering 
\includegraphics[width=0.35\linewidth]{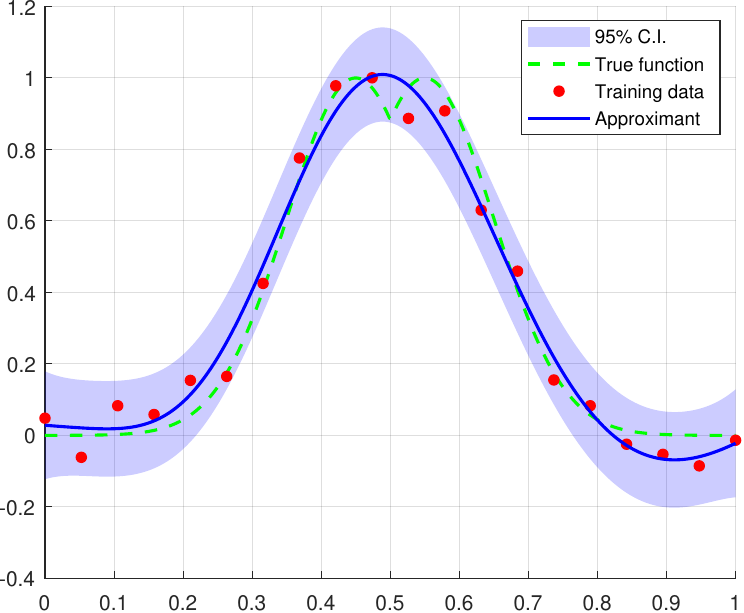} 
\hskip 0.5 cm
\includegraphics[width=0.35\linewidth]{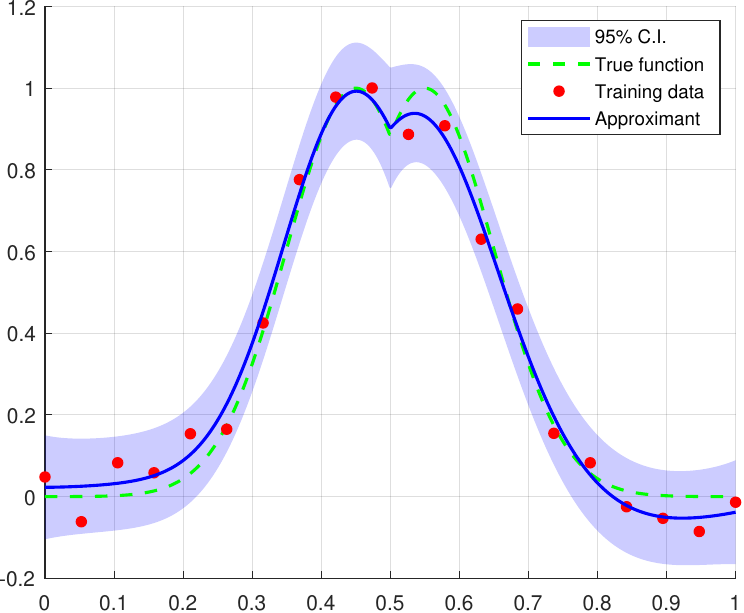} 
\caption{Stationary case (left) and VSK approach (right). The training data consist of $N=20$ uniform nodes.} \label{fig:approx_cusp_even} 
\end{figure}

\begin{figure}[H] 
\centering 
$
\begin{array}{c}
\includegraphics[width=0.35\linewidth]{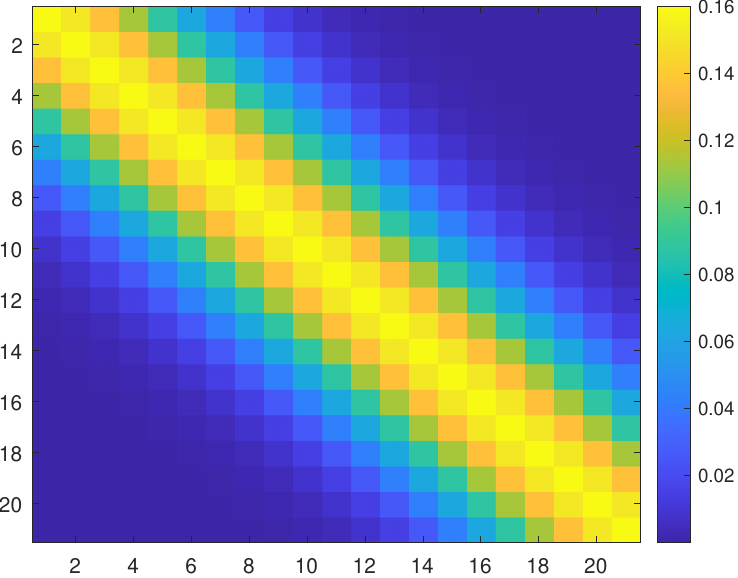} 
\hskip 0.5 cm
\includegraphics[width=0.35\linewidth]{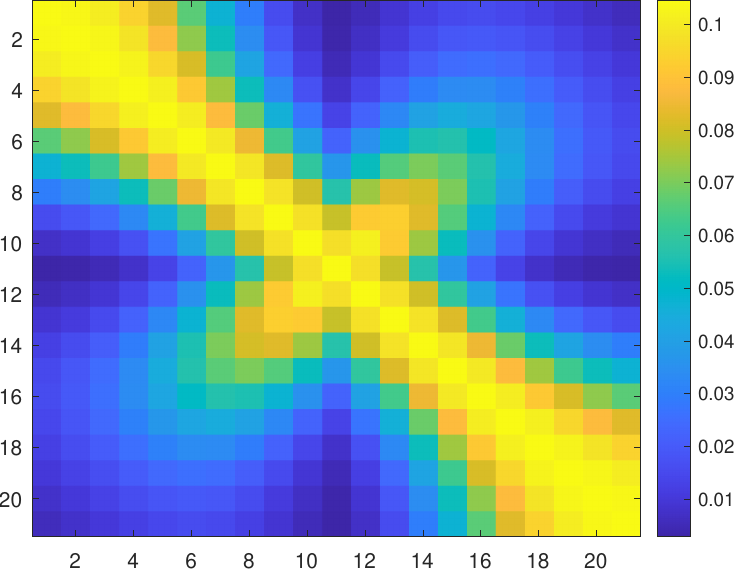} \\
\includegraphics[width=0.35\linewidth]{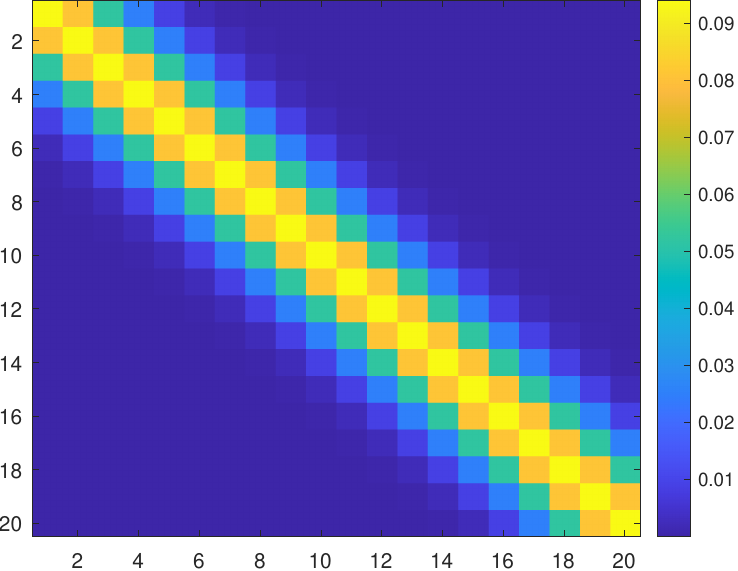} 
\hskip 0.5 cm
\includegraphics[width=0.35\linewidth]{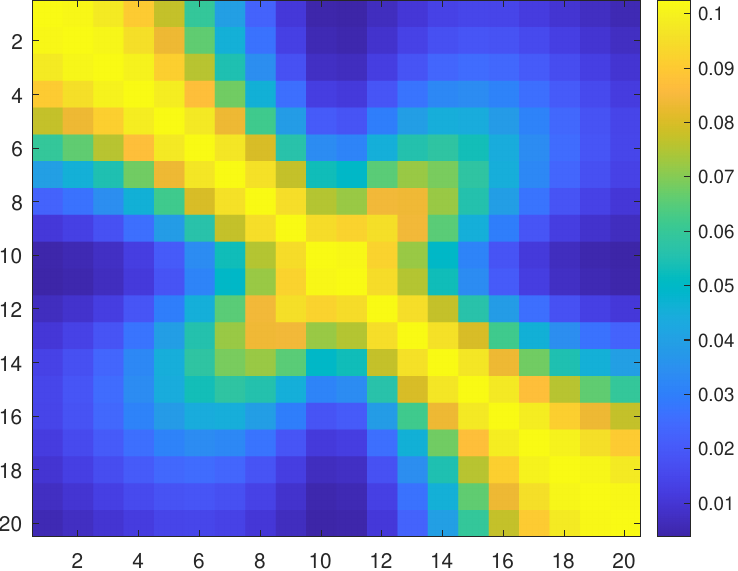} 
\end{array}
$
\caption{Training matrices in the stationary (left) and VSK (right) case: with $N=21$ (top) and $N=20$ (bottom) uniform nodes.} 
\label{fig:covariance_cusp} 
\end{figure}


\subsection{VSKs and other non-stationary kernels}\label{sec:results_others}

In this experiment, we investigate to what extent classical nonstationary Gaussian process kernels of Gibbs or Paciorek--Schervish type can approximate a VSK kernel beyond the local asymptotic regime established in Section~\ref{sec:theory}.
The goal is not to compare predictive performance, but rather to assess whether the theoretical local equivalence between VSK and Gibbs/Paciorek kernels can be leveraged in practice to reproduce the behavior of a VSK prior using standard nonstationary constructions.

We consider a one-dimensional regression problem on the interval $\Omega=[-1,1]$.
$N=9$ training points are chosen according to a Chebyshev distribution, and independent Gaussian noise with standard deviation $0.05    $ is added to the observations. The base kernel is the squared exponential kernel.

The target function is defined as $f(x) = \exp(x) - \cos(2\pi x)$,
which combines smooth exponential growth with moderate oscillatory behavior.
Both a stationary process and a VSK process are trained, with all hyperparameters estimated via {MLE}.

For the VSK model, the scaling map is chosen as $\psi(x) = f(x)$, so that the induced embedding incorporates global information about the geometry of the target function. Using the optimized VSK hyperparameters, we then construct a nonstationary Gaussian process with a Gibbs/Paciorek-type kernel. In one dimension and for the squared exponential base kernel, the Gibbs and Paciorek constructions coincide, and the local length scale is set according to the theoretical result, $\ell(x) = \frac{\ell}{\sqrt{1+(\psi'(x))^2}}$, as predicted by the local asymptotic analysis. Results in Figure \ref{fig:approx_fass} confirm that the VSK reconstruction is enhanced in case a precise prior is encoded, as observed in previous tests, however the reconstruction achieved by the Gibbs/Paciorek model with the defined $\ell$ shows a highly oscillating behavior that does not allow a precise reconstruction.

In addition to comparing the resulting interpolants, we then examine individual kernel basis functions centered at the central sampling node for the stationary, VSK, and Gibbs/Paciorek kernels.
This comparison, which is displayed in Figure \ref{fig:basis_fass}, reveals that the similarity between Gibbs/Paciorek and VSK kernels rapidly deteriorates away from the origin. The overlap between the corresponding basis functions is not notable besides the common value at the origin, indicating that the local equivalence does not extend to the global structure of the kernel. 
This numerical evidence indicates that VSKs encode different geometric information, which cannot be recovered globally through standard non-stationary approaches.

\begin{figure}[H] 
\centering 
\includegraphics[width=0.32\linewidth]{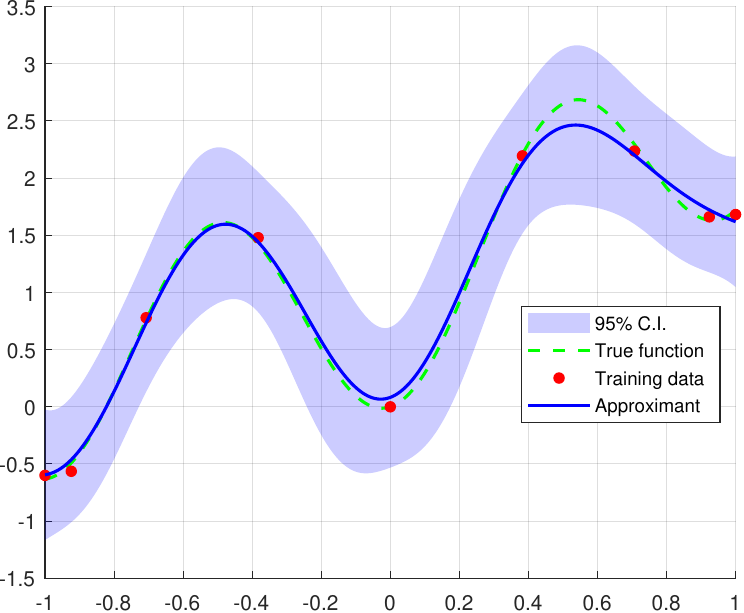}
\hskip 0.1 cm
\includegraphics[width=0.32\linewidth]{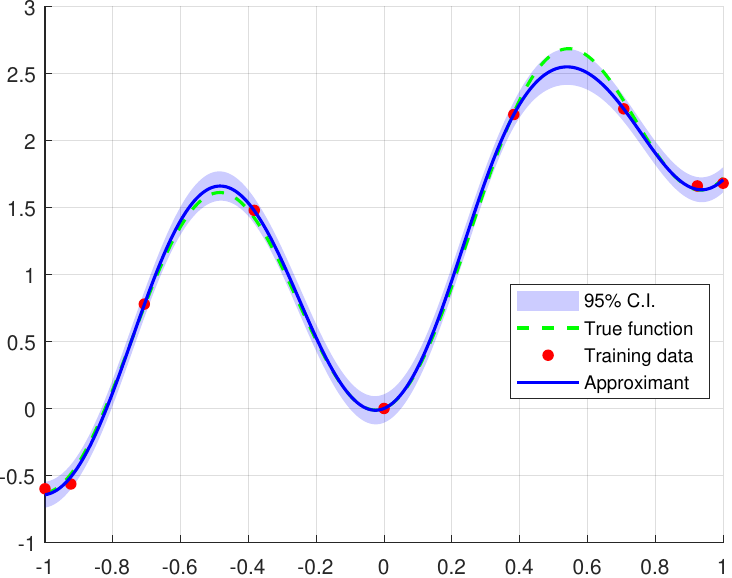} 
\hskip 0.1 cm
\includegraphics[width=0.32\linewidth]{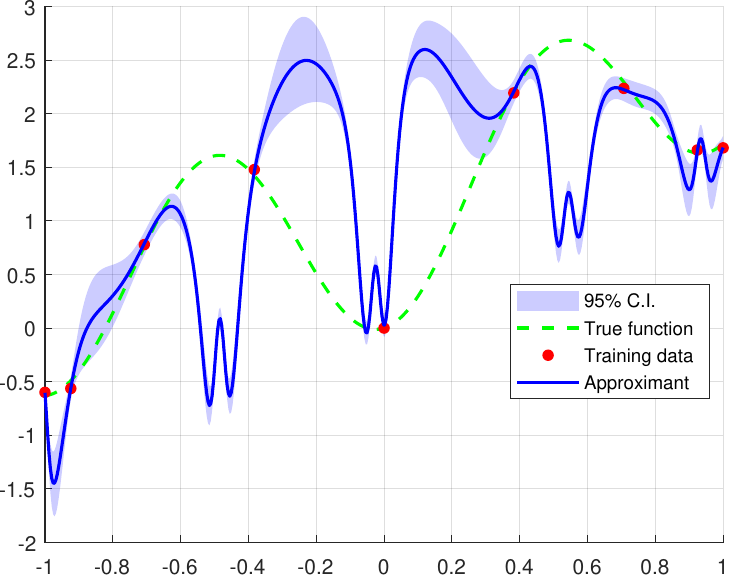} 
\caption{Stationary case (left), VSK approach (center) and a \textit{corresponding} Gibbs kernel (right). The training data consist of $N=9$ Chebyshev nodes.} \label{fig:approx_fass} 
\end{figure}

\begin{figure}[H] 
\centering 
\includegraphics[width=0.32\linewidth]{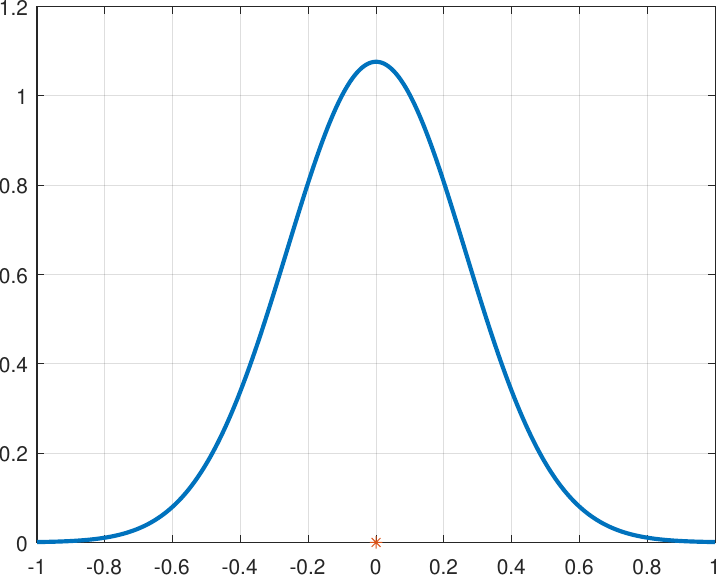}
\hskip 0.1 cm
\includegraphics[width=0.32\linewidth]{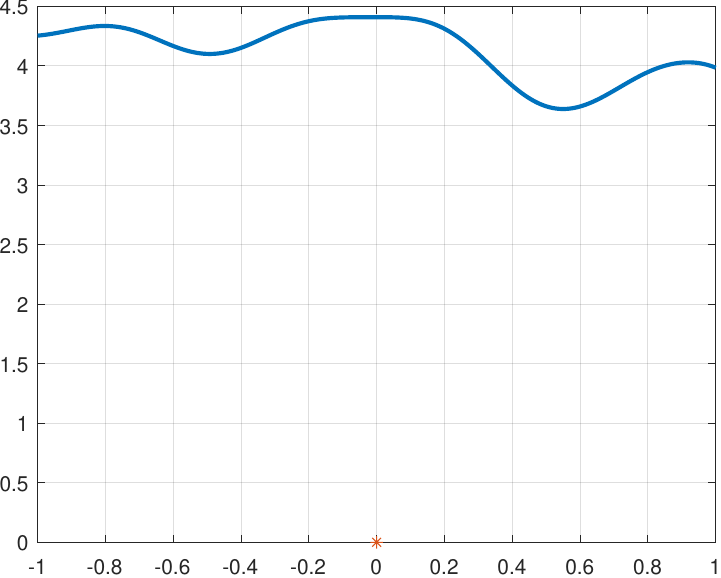} 
\hskip 0.1 cm
\includegraphics[width=0.32\linewidth]{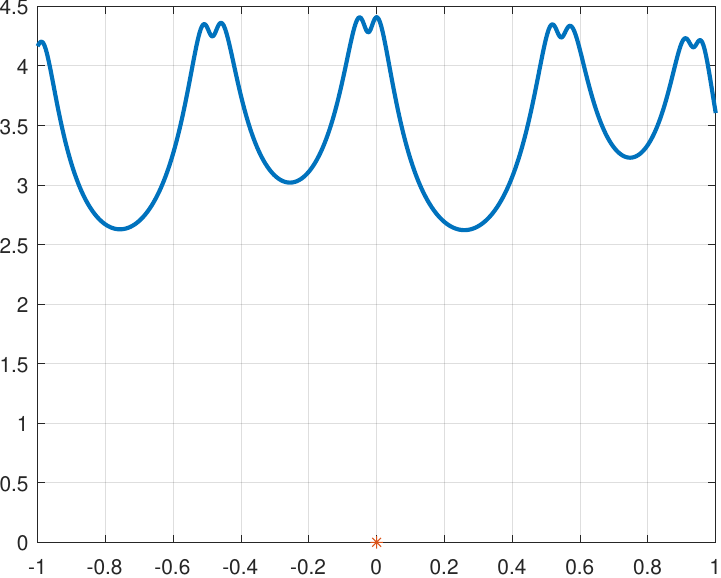} 
\caption{Basis functions centered at $x'=0$: stationary case (left), VSK approach (center) and a \textit{corresponding} Gibbs kernel (right).} 
\label{fig:basis_fass} 
\end{figure}

\section{Conclusions}
\label{conclusions}

This work provides a study of variably scaled kernels within the Kriging/Gaussian process framework and showed that the lifting induced by a scaling function generates a non-stationary covariance with a clear geometric interpretation. We highlighted the relationship between VSK-based prediction and classical kernel interpolation, and we also established explicit connections with existing non-stationary constructions.
The numerical experiments on discontinuous, corner-type, and highly irregular targets indicate that VSK-based models can yield more accurate reconstructions and uncertainty patterns that better reflect the structural information. At the same time, the comparison with Gibbs/Paciorek-type models suggests that the local asymptotic correspondence does not extend to a full global equivalence. Overall, VSKs provide a flexible and geometrically interpretable route to non-stationary Gaussian process modeling.




\bibliographystyle{abbrv}
\bibliography{refs_fin_doi}

@ARTICLE{Kozak2019,
	author = {Kozak, David and Holladay, Scott and Fasshauer, Gregory E.},
	title = {Intraday load forecasts with uncertainty},
	year = {2019},
	journal = {Energies},
	volume = {12},
	number = {10},
	doi = {10.3390/en12101833}
}

@ARTICLE{McCourt2017917,
	author = {McCourt, Michael and Fasshauer, Gregory E.},
	title = {Stable likelihood computation for gaussian random fields},
	year = {2017},
	journal = {Applied and Numerical Harmonic Analysis},
	pages = {917 – 943},
	doi = {10.1007/978-3-319-55556-0_16}
}

@article{scheuerer2013interpolation,
	author = {Scheuerer, M. and Schaback, R. and Schlather, M.},
	title = {Interpolation of spatial data-A stochastic or a deterministic problem?},
	year = {2013},
	journal = {European Journal of Applied Mathematics},
	volume = {24},
	number = {4},
	pages = {601 – 629},
	
	doi = {10.1017/S0956792513000016}
}

@book {wendland2005scattered,
    AUTHOR = {Wendland, H.},
     TITLE = {Scattered data approximation},
    SERIES = {Cambridge Monographs on Applied and Computational Mathematics},
    VOLUME = {17},
 PUBLISHER = {Cambridge University Press, Cambridge},
      YEAR = {2005},
     PAGES = {x+336},
      ISBN = {978-0521-84335-5; 0-521-84335-9},
   MRCLASS = {41-02 (41A10 41A63 65D10)},
  MRNUMBER = {2131724 (2006i:41002)},

	doi = {10.1017/CBO9780511617539}
}

@book{rasmussen2005gaussian,
	title = {Gaussian {Processes} for {Machine} {Learning}},
	copyright = {http://creativecommons.org/licenses/by-nc-nd/4.0/},
	isbn = {978-0-262-25683-4},
	doi = {10.7551/mitpress/3206.001.0001},
	language = {en},
	publisher = {The MIT Press},
	author = {Rasmussen, Carl Edward and Williams, Christopher K. I.},
	month = nov,
	year = {2005},
}

@book{stein1999interpolation,
	address = {New York, NY},
	series = {Springer {Series} in {Statistics}},
	title = {Interpolation of {Spatial} {Data}},
	copyright = {http://www.springer.com/tdm},
	isbn = {978-1-4612-7166-6 978-1-4612-1494-6},
	doi = {10.1007/978-1-4612-1494-6},
	publisher = {Springer New York},
	author = {Stein, Michael L.},
	year = {1999},
}

@book{berlinet2004rkhs,
	address = {Boston, MA},
	title = {Reproducing {Kernel} {Hilbert} {Spaces} in {Probability} and {Statistics}},
	copyright = {http://www.springer.com/tdm},
	isbn = {978-1-4613-4792-7 978-1-4419-9096-9},
	doi = {10.1007/978-1-4419-9096-9},
	language = {en},
	publisher = {Springer US},
	author = {Berlinet, Alain and Thomas-Agnan, Christine},
	year = {2004},
}

@article{aronszajn1950rkhs,
	title = {Theory of reproducing kernels},
	volume = {68},
	doi = {10.1090/S0002-9947-1950-0051437-7},
	language = {en},
	number = {3},
	journal = {Trans. Amer. Math. Soc.},
	author = {Aronszajn, N.},
	year = {1950},
	pages = {337--404},
}

@book{cressie1993statistics,
	edition = {1},
	series = {Wiley {Series} in {Probability} and {Statistics}},
	title = {Statistics for {Spatial} {Data}},
	isbn = {978-0-471-00255-0 978-1-119-11515-1},
	doi = {10.1002/9781119115151},
	language = {en},
	publisher = {Wiley},
	author = {Cressie, Noel A. C.},
	month = sep,
	year = {1993},
}

@inproceedings{goldberg1998heteroscedastic,
	title = {Regression with {Input}-dependent {Noise}: {A} {Gaussian} {Process} {Treatment}},
	volume = {10},
	booktitle = {Advances in {Neural} {Information} {Processing} {Systems}},
	publisher = {MIT Press},
	author = {Goldberg, Paul and Williams, Christopher and Bishop, Christopher},
	editor = {Jordan, M. and Kearns, M. and Solla, S.},
	year = {1997},

	doi = {10.5555/3008904.3008974}
}

@inproceedings{kersting2007mlhgp,
	address = {Corvalis Oregon USA},
	title = {Most likely heteroscedastic {Gaussian} process regression},
	isbn = {978-1-59593-793-3},
	doi = {10.1145/1273496.1273546},
	language = {en},
	booktitle = {Proceedings of the 24th international conference on {Machine} learning},
	publisher = {ACM},
	author = {Kersting, Kristian and Plagemann, Christian and Pfaff, Patrick and Burgard, Wolfram},
	month = jun,
	year = {2007},
	pages = {393--400},
}

@inproceedings{lazaro2011vhgp,
	address = {Madison, WI, USA},
	series = {{ICML}'11},
	title = {Variational heteroscedastic {Gaussian} process regression},
	isbn = {978-1-4503-0619-5},
	booktitle = {Proceedings of the 28th {International} {Conference} on {International} {Conference} on {Machine} {Learning}},
	publisher = {Omnipress},
	author = {Lázaro-Gredilla, Miguel and Titsias, Michalis K.},
	year = {2011},
	pages = {841--848},

	doi = {10.5555/3104482.3104588}
}

@article{bozzini2015interpolation,
	title = {Interpolation with variably scaled kernels},
	volume = {35},
	doi = {10.1093/imanum/drt071},
	language = {en},
	number = {1},
	journal = {IMA Journal of Numerical Analysis},
	author = {Bozzini, M. and Lenarduzzi, L. and Rossini, M. and Schaback, R.},
	month = jan,
	year = {2015},
	pages = {199--219},
}

@article{rossini2022overview,
	address = {IT},
	title = {Variably scaled kernels: an overview},
	volume = {15},
	shorttitle = {Variably scaled kernels},
	doi = {10.14658/pupj-drna-2022-4-6},
	language = {eng},
	number = {12/2022},
	journal = {Dolomites Research Notes on Approximation},
	publisher = {Padova University Press},
	author = {Rossini, Milvia},
	year = {2022},
	pages = {61--72},
}

@article{demarchi2019vsdk,
	title = {Jumping with variably scaled discontinuous kernels ({VSDKs})},
	volume = {60},
	doi = {10.1007/s10543-019-00786-z},
	language = {en},
	number = {2},
	journal = {{BIT} Numerical Mathematics},
	author = {De Marchi, S. and Marchetti, F. and Perracchione, E.},
	month = jun,
	year = {2020},
	pages = {441--463},
}

@ARTICLE{DeMarchi-et-al:2020,
author={De Marchi, S. and Erb, W. and Marchetti, F. and Perracchione, E. and Rossini, M.},
title={Shape-driven interpolation with discontinuous kernels: Error analysis, edge extraction, and applications in magnetic particle imaging},
journal={SIAM Journal on Scientific Computing},
year={2020},
volume={42},
number={2},
pages={B472-B491},

	doi = {10.1137/19M1248777}
}

@ARTICLE{campi2021vsk,
author={Campi, C. and Marchetti, F. and Perracchione, E.},
title={Learning via variably scaled kernels},
journal={Advances in Computational Mathematics},
year={2021},
volume={47},
number={4},

	doi = {10.1007/s10444-021-09875-6}
}

@inproceedings{paciorek2004nonstationary,
	title = {Nonstationary {Covariance} {Functions} for {Gaussian} {Process} {Regression}},
	volume = {16},
	booktitle = {Advances in {Neural} {Information} {Processing} {Systems}},
	publisher = {MIT Press},
	author = {Paciorek, Christopher and Schervish, Mark},
	editor = {Thrun, S. and Saul, L. and Schölkopf, B.},
	year = {2003},

	doi = {10.5555/2981345.2981380}
}

@article{paciorek2006environmetrics,
	title = {Spatial modelling using a new class of nonstationary covariance functions},
	volume = {17},
	copyright = {http://onlinelibrary.wiley.com/termsAndConditions\#vor},
	doi = {10.1002/env.785},
	language = {en},
	number = {5},
	journal = {Environmetrics},
	author = {Paciorek, Christopher J. and Schervish, Mark J.},
	month = aug,
	year = {2006},
	pages = {483--506},
}

@ARTICLE{Rossini:2018,
author={Rossini, M.},
title={Interpolating functions with gradient discontinuities via variably scaled kernels},
journal={Dolomites Research Notes on Approximation},
year={2018},
volume={11},
pages={3--14},
	doi = {10.14658/PUPJ-DRNA-2018-2-2}
}

@inproceedings{wilson2013sm,
	title = {Gaussian process kernels for pattern discovery and extrapolation},
	number = {PART 3},
	booktitle = {30th {International} {Conference} on {Machine} {Learning}, {ICML} 2013},
	author = {Wilson, Andrew Gordon and Adams, Ryan Prescott},
	year = {2013},
	pages = {2104 -- 2112},
	doi = {10.5555/3042817.3043056}
}

@book{Bernstein,
	address = {Princeton, N.J},
	edition = {2nd ed},
	title = {Matrix mathematics: theory, facts, and formulas},
	isbn = {978-0-691-13287-7},
	shorttitle = {Matrix mathematics},
	language = {eng},
	publisher = {Princeton university press},
	author = {Bernstein, Dennis S.},
	year = {2009},

	doi = {10.1515/9781400833344}
}

@article{gramacylee2008tgp,
	title = {Bayesian {Treed} {Gaussian} {Process} {Models} {With} an {Application} to {Computer} {Modeling}},
	volume = {103},
	doi = {10.1198/016214508000000689},
	language = {en},
	number = {483},
	journal = {Journal of the American Statistical Association},
	author = {Gramacy, Robert B and Lee, Herbert K. H},
	month = sep,
	year = {2008},
	pages = {1119--1130},
}

@article{higdon1998processconv,
  title={A Process-convolution Approach to Modelling Temperatures in the North Atlantic Ocean},
  author={Higdon, Dave},
  journal={Environmental and Ecological Statistics},
  volume={5},
  pages={173--190},
  year={1998},
	doi = {10.1023/A:1009666805688}
}

@article{sampson1992nonparametric,
  title={Nonparametric Estimation of Nonstationary Spatial Covariance Structure},
  author={Sampson, Paul D and Guttorp, Peter},
  journal={Journal of the American Statistical Association},
  volume={87},
  number={417},
  pages={108--119},
  year={1992},
	doi = {10.1080/01621459.1992.10475181}
}

@article{schmidt2003deformations,
  title={Bayesian Inference for Non-stationary Spatial Covariance Structure via Spatial Deformations},
  author={Schmidt, Alexandra M and O'Hagan, Anthony},
  journal={Journal of the Royal Statistical Society: Series B},
  volume={65},
  number={3},
  pages={743--758},
  year={2003},
	doi = {10.1111/1467-9868.00413}
}

@article{priestley1965,
  title={Evolutionary Spectra and Non-stationary Processes},
  author={Priestley, Maurice B},
  journal={Journal of the Royal Statistical Society: Series B},
  volume={27},
  number={2},
  pages={204--229},
  year={1965},
	doi = {10.1111/j.2517-6161.1965.tb01488.x}
}

@article{silverman1957,
  title={On the Estimation of a Probability Density Function by the Maximum Penalized Likelihood Method},
  author={Silverman, Bernard W},
  journal={Annals of Statistics},
  volume={10},
  number={3},
  pages={795--810},
  year={1982},
	doi = {10.1214/aos/1176345872}
}

@book{berman1979nonnegative,
author = {Berman, Abraham and Plemmons, Robert J.},
title = {Nonnegative Matrices in the Mathematical Sciences},
publisher = {Society for Industrial and Applied Mathematics},
year = {1994},
doi = {10.1137/1.9781611971262},
address = {},
edition   = {},
}

@article{borovitskiy2020matern,
  title={Mat{\'e}rn Gaussian Processes on Riemannian Manifolds},
  author={Borovitskiy, Viacheslav and Terenin, Alexander and Mostowsky, Peter and Deisenroth, Marc Peter},
  journal={Advances in Neural Information Processing Systems},
  volume={33},
  pages={12426--12437},
  year={2020},
	doi = {10.5555/3495724.3496766}
}

@incollection {smola-et-al:2001,
    AUTHOR = {Sch\"olkopf, Bernhard and Herbrich, Ralf and Smola, Alex J.},
     TITLE = {A generalized representer theorem},
 BOOKTITLE = {Computational learning theory ({A}msterdam, 2001)},
    SERIES = {Lecture Notes in Comput. Sci.},
    VOLUME = {2111},
     PAGES = {416--426},
 PUBLISHER = {Springer, Berlin},
      YEAR = {2001},
      ISBN = {3-540-42343-5},
   MRCLASS = {68T05},
  MRNUMBER = {2042050},
	doi = {10.1007/3-540-44581-1_27}
}

@article{Schaback_Wendland_2006, 
author = {Schaback, Robert and Wendland, Holger},
	title = {Kernel techniques: From machine learning to meshless methods},
	year = {2006},
	journal = {Acta Numerica},
	volume = {15},
	pages = {543 – 639},
	doi = {10.1017/S0962492906270016}
}

@article{DeMarchi-et-al:2019,
  author  = {De Marchi, Stefano and Martínez, A{ngeles} and Perracchione, Emma and Rossini, Milvia},
  title   = {RBF-Based Partition of Unity Methods for Elliptic PDEs: Adaptivity and Stability Issues Via Variably Scaled Kernels},
  journal = {Journal of Scientific Computing},
  year    = {2019},
  volume  = {79},
  number  = {1},
  pages   = {321--344},
  doi     = {10.1007/s10915-018-0851-2}
}

@book{Fasshauer-et-al:2015,
  title={Kernel-based approximation methods using Matlab},
  author={Fasshauer, Gregory E and McCourt, Michael J},
  volume={19},
  year={2015},
  publisher={World Scientific Publishing Company},
	doi = {10.1142/9335}
}

\section*{Appendix}

Proof of Theorem \ref{prop:vsk-local-symmetric}.
\begin{proof}
The VSK covariance is defined in \eqref{eq:vsk-def:cov}, where for $x,x' \in \Omega:$
$$\kappa^\Psi_\ell(x,x')
=
\varphi_\ell\!\left(d_\Psi(x,x')\right),$$
being
\begin{equation}\label{dpsi}
d_\Psi(x,x')
=
\sqrt{\|x-x'\|_2^2 + (\psi(x)-\psi(x'))^2}.
\end{equation}

Let $h=x'-x$ with $\|h\|_2=o(1)$. Since $\psi\in C^1(\Omega)$, Taylor expansion yields
\[
\psi(x+h)
=
\psi(x)
+
\nabla\psi(x)^\intercal h
+
o(\|h\|_2). \quad h\to 0.
\]
Therefore,
\[
(\psi(x)-\psi(x'))^2
=
(\nabla\psi(x)^\intercal h)^2
+
{o(\|h\|_2^2),\qquad h\to 0.}
\]
Substituting into \eqref{dpsi}, we get
\begin{align*}
d_\Psi(x,x')^2
&=
\|h\|_2^2 + (\nabla\psi(x)^\intercal h)^2 + o(\|h\|_2^2)
\nonumber\\
&=
h^\intercal\!\left(I+\nabla\psi(x)\nabla\psi(x)^\intercal\right)h
+{o(\|h\|_2^2),\qquad h\to 0.}
\end{align*}
Repeating the same expansion with the roles of $x$ and $x'$ exchanged yields
\[
d_\Psi(x,x')^2
=
h^\intercal\!\left(I+\nabla\psi(x')\nabla\psi(x')^\intercal\right)h
+
o(\|h\|_2^2),\qquad h\to 0.
\]
Averaging the two expressions produces the symmetric expansion
\[
d_\Psi(x,x')^2
=
h^\intercal \overline{M}(x,x') h
+
o(\|h\|_2^2),\qquad h\to 0,
\]
where $\overline{M}(x,x')$ is defined in \eqref{eq:Mbar-def}. 

Since $\overline{M}(x,x')$ is symmetric positive definite and depends continuously on
$x,x'$, there exists a neighborhood of the diagonal $\{(x,x'):\|x-x'\|_2<\rho\}$
and constants $0<c\le C<\infty$ such that
\[
c\|h\|_2^2
\le
h^\top \overline{M}(x,x')h
\le
C\|h\|_2^2.
\]
We have that
$d_\Psi(x,x')^2=h^\top \overline{M}(x,x')h+o(\|h\|_2^2)$.
Then
\[
d_\Psi(x,x')-\sqrt{h^\top \overline{M}(x,x')h}
=
\frac{o(\|h\|_2^2)}{d_\Psi(x,x')+\sqrt{h^\top \overline{M}(x,x')h}},
\]
and,
\[
d_\Psi(x,x')+\sqrt{h^\top \overline{M}(x,x')h}
\ge
\sqrt{h^\top \overline{M}(x,x')h}
\ge
\sqrt{c}\,\|h\|_2.
\]
Therefore,
\[
\bigl|d_\Psi(x,x')-\sqrt{h^\top \overline{M}(x,x')h}\bigr|
\le
\frac{|o(\|h\|_2^2)|}{\sqrt{c}\,\|h\|_2}
=
o(\|h\|_2),
\qquad h\to 0,
\]
and we conclude that
\[
d_\Psi(x,x')
=
\sqrt{h^\top \overline{M}(x,x')h}
+
o(\|h\|_2),
\qquad x'\to x.
\]

$\kappa^\Psi_\ell(x,x')=\varphi(d_\Psi(x,x')/\ell)$ yields
\[
\varphi_\ell(d_\Psi(x,x'))
=
\varphi\!\left(\frac{\sqrt{h^\top \overline{M}(x,x')h}}{\ell}\right)
+
o(1),
\qquad x'\to x,
\]
since $\varphi$ is continuous at $0$ and $d_\Psi(x,x')\to 0$ as $x'\to x$.
Finally,
\[
\phi^\Psi_{\ell,f,n}(x,x')
=
\sigma_f^2\,
\varphi\!\left(
\frac{\sqrt{h^\top \overline{M}(x,x')h}}{\ell}
\right)
+
\sigma_n^2\delta_{xx'}
+
o(1),
\qquad x'\to x,
\]
which proves the claim.

\end{proof}

Proof of Corollary \ref{cor:vsk-d1}.
\begin{proof}
Being $d=1$ for $x,x' \in \Omega$, 
$\overline{M}(x,x')$ can be written as
\[
\overline{M}(x,x')
=
1+\frac{1}{2}\big((\psi'(x))^2+(\psi'(x'))^2\big).
\]
Thus,
\[
\sqrt{h^\intercal \overline{M}(x,x') h}
=
|x-x'|
\sqrt{1+\frac{(\psi'(x))^2+(\psi'(x'))^2}{2}}.
\]
Since $\psi\in C^1(\Omega)$, we have
$\psi'(x')=\psi'(x)+O(|x'-x|)$, hence
$(\psi'(x'))^2=(\psi'(x))^2+O(|x'-x|)$.
Letting $a=(\psi'(x))^2$ and $b=(\psi'(x'))^2$, so that
$b=a+\delta$ with $\delta=O(|x'-x|)$.

Then
\[
\sqrt{1+\frac{a+b}{2}}
=
\sqrt{1+a+\frac{\delta}{2}}=\sqrt{1+a}\sqrt{1+\frac{\delta}{2(1+a)}}.
\]
Using the Taylor expansion at $t=0$: $\sqrt{1+t}=1+t/2+O(t^2)$, we get
\begin{equation}\label{eq:parte_1}
\sqrt{1+a}\sqrt{1+\frac{\delta}{2(1+a)}}=
\sqrt{1+a}
\left(
1+\frac{\delta}{4(1+a)}+{O(\delta^2),\quad \delta\to0}
\right) 
\end{equation}

On the other hand, by definition of $\tilde{\ell}^{\Psi}$,
\begin{align*}
\tilde{\ell}^{\Psi}(x,x')^2
& = 
\frac{\ell^2}{2}
\left(
\frac{1}{1+a}+\frac{1}{1+b}
\right) \\ &= 
\frac{\ell^2}{2}
\left(
\frac{1}{1+a}+\frac{1}{1+a+\delta}
\right)
\\ &= 
\frac{\ell^2}{2}
\left(
\frac{1}{1+a}+\frac{1}{1+a}\frac{1}{1+\frac{\delta}{1+a}}
\right).
\end{align*}

Using the Taylor expansion at $t=0$: $1/(1+t)=1-t+O(t^2)$, we get
\begin{align*}
& \frac{\ell^2}{2}
\left(
\frac{1}{1+a}+\frac{1}{1+a}\frac{1}{1+\frac{\delta}{1+a}}
\right)=\\&
\frac{\ell^2}{2}
\left(
\frac{1}{1+a}+\frac{1}{1+a}\left(1-\frac{\delta}{1+a}+O(\delta^2)\right)
\right),
\end{align*}
and thus
\begin{align*}
&
\frac{\ell^2}{2}
\left(
\frac{1}{1+a}+\frac{1}{1+a}\left(1-\frac{\delta}{1+a}+O(\delta^2)\right)
\right)
=\\&
\frac{\ell^2}{1+a}
\left(
1-\frac{\delta}{2(1+a)}+{O(\delta^2)}
\right),
\end{align*}

Taking square roots and using the Taylor expansion
$\sqrt{1+t}=1+\tfrac{t}{2}+O(t^2)$ as $t\to0$ yields
\[
\tilde{\ell}^{\Psi}(x,x')
=
\frac{\ell}{\sqrt{1+a}}
\left(
1-\frac{\delta}{4(1+a)}
{O(\delta^2)}
\right).
\]

Inverting this expression and using the expansion
$(1-t)^{-1}=1+t+{O(t^2)}$ gives
\begin{equation}\label{eq:parte_2}
\frac{\ell}{\tilde{\ell}^{\Psi}(x,x')}
=
\sqrt{1+a}
\left(
1+\frac{\delta}{4(1+a)}+{O(\delta^2)}
\right).
\end{equation}

Combining the two expansions \eqref{eq:parte_1} and \eqref{eq:parte_2} yields 
\[
\sqrt{1+\frac{(\psi'(x))^2+(\psi'(x'))^2}{2}}
=
\frac{\ell}{\tilde{\ell}^{\Psi}(x,x')}
+
{o(1),\quad x'\to x}.
\]

As explicitly shown in the previous proof, since $\tilde{\ell}^{\Psi}(x,x')$ is continuous and $\varphi$ is continuous at
the origin, the $o(1)$ perturbation of the argument is preserved under
composition with $\varphi$. 
\end{proof}

Proof of Corollary \ref{prop:vsk-ps}.
\begin{proof}
Since $\psi\in C^1(\Omega)$, the map $x\mapsto \nabla\psi(x)$ is continuous.
Hence $x\mapsto M(x)$ is continuous and, in particular,
\[
M(x')=M(x)+o(1),
\qquad x'\to x.
\]
Since $M(x)$ is symmetric positive definite for all $x$, the matrix inversion
map is continuous on the set of SPD matrices, and therefore
\[
M(x')^{-1}=M(x)^{-1}+o(1),
\qquad x'\to x.
\]

Using \eqref{eq:Sigma-vsk}, we obtain
\[
\frac{1}{2}\big(\Sigma(x)+\Sigma(x')\big)
=
\frac{\ell^2}{2}\big(M(x)^{-1}+M(x')^{-1}\big)
=
\ell^2\left(M(x)^{-1}+o(1)\right),
\qquad x'\to x.
\]
Taking inverses and using continuity of the inverse map on SPD matrices yields
\[
\Big(\tfrac{1}{2}(\Sigma(x)+\Sigma(x'))\Big)^{-1}
=
\frac{1}{\ell^2}\left(M(x)+o(1)\right),
\qquad x'\to x.
\]
Repeating the same argument with the roles of $x$ and $x'$ exchanged gives
\[
\Big(\tfrac{1}{2}(\Sigma(x)+\Sigma(x'))\Big)^{-1}
=
\frac{1}{\ell^2}\left(M(x')+o(1)\right),
\qquad x'\to x.
\]
Averaging the two expressions yields
\[
\Big(\tfrac{1}{2}(\Sigma(x)+\Sigma(x'))\Big)^{-1}
=
\frac{1}{\ell^2}\left(\overline{M}(x,x')+o(1)\right),
\qquad x'\to x.
\]

Finally, multiplying on the left and right by $h^\intercal$ and $h$ gives
\[
h^\intercal
\Big(\tfrac{1}{2}(\Sigma(x)+\Sigma(x'))\Big)^{-1}
h
=
\frac{1}{\ell^2}\,h^\intercal \overline{M}(x,x')h
+
h^\intercal o(1)\, h.
\]
Since $h^\intercal o(1)\,h = o(\|h\|_2^2)$, this proves \eqref{eq:ps-local-limit}.
\end{proof}

\end{document}